\documentclass[a4wide]{article}

\usepackage{authblk}
\usepackage{eurosym}
\usepackage{amsmath,amssymb,stmaryrd,pifont,wasysym}
\usepackage{color}
\usepackage{tikz}
\usetikzlibrary{shapes,calc,automata}
\usepackage{graphicx}
\usepackage[ruled,vlined,linesnumbered]{algorithm2e}
\usepackage{pgfplots}
\usepackage{caption}
\usepackage{MnSymbol}

\def\ent { \mathbb{N} }
\def\reel { \mathbb{R} }

\def\rat { \mathbb{Q} }

\newtheorem{example}{Example}

\newtheorem{definition}{Definition}
\newtheorem{theorem}{Theorem}
\newtheorem{remark}{Remark}

\newtheorem{corollary}{Corollary}
\newenvironment{proof}{\noindent\emph{Proof.}}{\hfill $\Box$\break\par}

\newcommand{\refe}[1] {(\ref{#1})}

\def \critun {Commute\, time}
\def \critdeux {Sport\, club} 
\def \crittrois {Size}
\def \critquatre {Cost}

\def \AA {-15}
\def \BB {gym}
\def \CC {400}
\def \DD {-5000}
\def \aa {-50}
\def \bb {no\, gym}
\def \cc{200}
\def \dd {-12000}

\def \A {A}
\def \B {B}
\def \C {C}
\def \D {D}

\def \a {a}
\def \b {b}
\def \c {c}
\def \d {d}

\begin{document}

\title{Explaining robust additive utility models by sequences of preference swaps
}


\author{K. Belahcene%
\thanks{Email: \texttt{Khaled.Belahcene@polytechnique.org}}}
\affil{LGI, Ecole Centrale de Paris, Chatenay Malabry, France}

\author{C. Labreuche%
\thanks{Email: \texttt{christophe.labreuche@thalesgroup.com}}}
\affil{Thales Research \& Technology, 91767 Palaiseau Cedex, France}

\author{N.  Maudet%
\thanks{Email: \texttt{nicolas.maudet@lip6.fr}}}
\affil{LIP6, Universit\'e Paris-6, 75252 Paris Cedex 05, France}

\author{V. Mousseau%
\thanks{Email: \texttt{vincent.mousseau@ecp.fr}}}
\affil{LGI, Ecole Centrale de Paris, Chatenay Malabry, France}

\author{W. Ouerdane%
\thanks{Email: \texttt{wassila.ouerdane@ecp.fr}}}
\affil{LGI, Ecole Centrale de Paris, Chatenay Malabry, France}


%

\date{08/02/2015}

\maketitle

\begin{abstract}

Multicriteria decision analysis aims at supporting a person facing a
decision problem involving conflicting criteria. We consider an 
additive utility model which provides
robust conclusions based on preferences elicited from the
decision maker. The recommendations based on these robust
conclusions are even more convincing if they are complemented by
explanations. We propose a general scheme, based on sequence
of preference swaps, in which explanations can be computed. We show first that the length of
explanations can be unbounded in the general case.   However, in the case of binary
reference scales, this  length is bounded and we provide an
algorithm to compute the corresponding explanation.

\end{abstract}

\section{Introduction}

\label{sec_intro}

A multi-criteria problem consists in formalizing the problem and eliciting the preferences of the decision maker (DM). 
In many decision contexts  providing recommendations based only on 
the elicited  preference model is insufficient. In fact, decision makers may want explanations which justify in a convincing way such recommendations. Indeed,
justifying and explaining a rationale for a decision is almost as important as
the recommendation itself. This is particularly true in situations where the
decision needs to be justified to some other stakeholders (who did not participate to the decision process). In this case, it is not satisfactory to just
present the preference model and the resulting recommendation. Although
technically this model does contain all the information on which such a recommendation is based, the format is unlikely to be suitable for presentation. Hence, the need for a synthetic and short explanation. Building a convincing
explanation is also required when the DM cannot be assumed to have any
mathematical background, as in the case of online recommender systems, where it has been shown that explanations improve the acceptability of the
recommended choice \cite{Pu2007}.

The problem of constructing or providing convincing explanations in order to justify recommended decisions has a rather long history in Artificial
Intelligence (see for instance \cite{caremoor06,klein94,SymeonidisRecSys09,Sullivan2007}). In a nutshell, the idea is to provide supporting evidence that a recommendation is
justified. This evidence may emphasize some critical data used for the recommendation, and/or provide a simplified version of the process which lead to
the recommendation \cite{HerlockerKR2000,FriedrichZankerAIMag2011}. However, the problem is especially difficult
in the context of multiple criteria models \cite{LabreucheAIJ2011,labmaudouer12}, where different
criteria are at stake.

In this paper we shall concentrate on the simple additive utility model. This well-known model assumes independence among criteria, although of course different criteria may have different weights. In other words, no synergy (either positive or negative) occurs between the different criteria.  An example of an elicitation method that relies on such a model, is the  \emph{even swaps} method \cite{hamkeerai98}. It aims at identifying, between two options $x$ and $y$, which one is preferred to another one without explicitly  constructing the utility functions. This is basically an elimination process based on trade-offs between \emph{pairs} of attributes (hence the name \emph{even swaps}). 
Broadly speaking, in such a swap, the DM changes the consequence (or score) of an alternative on one attribute, and compensates this
change with one another attribute, so that the new alternative is equally preferred in the end. 
By replacing one option (say $x$) with a different but equally preferred one, the hope is that dominance will occur over $y$. The process is thus repeated until dominance can be shown to hold, allowing to progressively eliminate attributes.

The \emph{UTA (Additive UTility)} method is another way to elicit the preference of the decision maker and  to construct a recommendation \cite{JacquetLagrezeSiskos82}. The DM is supposed to provide a set of comparison of alternatives. The advantage of this approach over the even swaps method is that the required preferential information is simpler from the DM point of view. Trade-offs are indeed very complex to provide. In fact, as it was pointed in \cite{hamkeerai98}: ``Making wise trade-offs is one of the most important and difficult challenges in decision making''.
On the other hand, the main drawback of the UTA method is that the preferential information brought by the DM is not sufficient to uniquely specify the utility functions (utility functions are only partially known), and UTA does not account for the multiplicity of the compatible utility. More precisely, this fact is dealt by choosing an arbitrary completion. 
As a matter of fact,
a conservative approach consists in relying on a robust (or necessary) preference relation \cite{grmasl08,grslfimo09}. In words, the relation holds if \emph{any} possible completion of the available preferential information yields the preferential statement.  

The aim of the paper is to explain a robust preference relation.
When utilities are known, one can interpret them  in terms of the importance of criteria and the satisfaction level of criteria \cite{klein94,caremoor06,LabreucheAIJ2011}.
This approach is not possible with the robust preference relation as the robust preference of an option $x$ over another one $y$ is complex.
The idea is to decompose a robust preference into several simpler recommendations, as it is the case in the even swaps approach.

However, the even swaps approach suffers from a limitation: by requiring each new generated option to be equally preferred to the initial one, this makes the technique poorly adapted to the context of incomplete preferences where such  an equivalence virtually never hold. 
The generalization of even swaps to robust relations is thus called \emph{preference swaps}. 
A very interesting property of even-swaps (for explanations) to justify a recommendation is that it does not require explicitly  the value of utility function or the model used to get the solution. 

Therefore, we construct a sequence of alternatives in which the first option is $x$ and the last one is $y$, and there is a robust preference between any two successive elements in the sequence. 
The existence of such a sequence entails the necessary preference of $x$ over $y$, by transitivity of the robust preference relation. 
The options in the sequence are constructed in such a way that each comparison in the sequence is simple to understand for the DM.
This is in particular the case if two successive alternatives in the sequence are similar on most of criteria and differ only on a few criteria.
This sequence of alternatives provides an explanation of the robust preference of $x$ over $y$.

The remainder of the paper is as follows.  Section \ref{sec_model}  describes the background notions and concepts that we  shall use in formulating the explanation.  Section \ref{sec_charac} presents new results concerning the necessary preference relation. These results are used to derive properties of the explanation engine introduced in this paper.  This engine is described and discussed in Section \ref{sec_explaining}, in terms of some technical characteristics that we believe important to deal with the question of constructing a formal explanation.  Finally,  in Section \ref{sec_binary} we address the case of binary preference scales, and show the properties of explanations when preferences are expressed in such scales.

\section{Background Notions and Literature review}

\label{sec_model}

We consider a finite set $N=\{1,\ldots, n\}$ of criteria. Each
criterion $i\in N$ is described by an attribute $X_i$. We assume
that all attributes are numerical and increasing. For discrete attributes, $X_i$
represents integers. For continuous attributes, $X_i$ is an
interval (possibly unbounded). Alternatives are considered as
elements of the Cartesian product of the attributes: $X= X_1\times
\cdots\times X_n$.

In this section we recall the principles of the standard additive utility model, see section \ref{sec_model1}. 
Then section \ref{sec_model2} describes shortly the even swaps approach.
Lastly section \ref{sec_model3} explains the robust preference relation.  Moreover, in order to get an understanding of the notions and different proposals of this paper we shall use the Example \ref{ex1}.  The idea is to provide, at the end,  an explanation for why an option is the best choice for the case considered.

\begin{example}
\label{ex1}

You need to rent an office for
your business  and your are undecided between two options, namely: $x$ and $y$.  
Such options are evaluated according to four criteria: \{ 1: the \critun (min); 2: 
the availability of a \critdeux ($gym, no\, gym$); 3:  the  \crittrois (m$^{\small 2}$),  and 4: the \critquatre
(\euro{}) \}.  Of course you want to minimize the cost and commute time, while you seek to maximize the availability of a sport club and the size.  \\

The criteria are described respectively by the following  increasing attributes: $X_1\subset \mathbb{R}^-, X_2=\{Yes, No\}, X_3 \subset \mathbb{R}^+,  and \, X_4 \subset \mathbb{R}^-$.   Thus,  the evaluation of the  previous  options is as follows:\\ 

$x=(-45, no\, gym, 450, -5000)$  and  $y=(-15, gym, 180, -12500)$\\
\end{example}

\subsection{Utility based preference model}
\label{sec_model1}

According to the additive utility model the comparison of two  multi-attribute
alternatives $x = (x_1,x_2,\cdots, x_n)$ and $y = (y_1,y_2,\cdots, y_n)$ is given by:
\begin{eqnarray*}
   \lefteqn{}x \succsim y & \Leftrightarrow &  \sum_{i\in N} u_i(x_i)
\ge \sum_{i\in N} u_i(y_i) 
\end{eqnarray*}
where the utility functions $u_i:X_i\rightarrow \reel$ are supposed to be known (and already elicited). This model entails separability, expressed by the ceteris paribus principle: when comparing two options, the criteria where the options have equal attributes do not count and can be merely ignored. In this paper,  the wildcard ``$\star$'' will denote the common evaluation when comparing two alternatives. For instance $(\AA, \bb, \cc, \star) \succsim (\aa, \BB, \cc, \star)$ means that the preference holds whatever the common evaluation on the fourth attribute.

A popular method to construct an additive utility model is the UTA (UTility Additive) approach \cite{JacquetLagrezeSiskos82}.
In this method, a subset $V_i$ of $X_i$ is chosen for each attribute.
The elements of $V_i$  are denoted by $V_i = \{v_{i,1} < v_{i,2} <
\ldots < v_{i,p_i}\} \subset X_i$, where $p_i = |V_i |$.  For instance, as it is described in the Example \ref{ex2}, for the first attribute, \critun, $V_1= \{\a:=\aa < \A:=\AA\}$. 

We set $V=V_1 \times \ldots \times V_n$. 
The unknowns of the decision model are the values of utility functions at the points in $V_1\times\cdots\times V_n$,
i.e. the values $(u_i(v_{i,1}),\ldots,u_i(v_{i,p_i}))_{i\in N}$.

Given two alternatives $x,y \in X$, the \emph{query} ``is $x$
preferred to $y$ ?" is denoted $(x \succeq_? y)$, and will be
given a truth value by determining the pair $(x,y)$'s membership
of various preference relations over $X$, which are partial
preorders : reflexive and transitive binary relations, but not
necessarily antisymmetric, and generally not complete\footnote{A related question is to count the number of queries to come up with a complete linear order \cite{FishburnPR02}}. 
For the moment, we consider two preference relations:

\begin{itemize}
\item $ ~\mathcal{D}$:  the
Pareto dominance relation (in Example \ref{ex2}: $\mathcal{D}= \{ (e_2 \succsim e_3)\}$ );

\item $\mathcal{P}$: the preferential
information explicitly given by the decision maker, leveraged as a
learning set (in Example \ref{ex2}:  $\mathcal{P}= \{ (e_1 \succsim e_2); (e_3 \succsim e_4), (e_4 \succsim e_5)\}$). 
In the UTA method, the preferential information $\mathcal{P}$ is composed of preference statements of the form $x \succsim y$, meaning that $x$ is at least as good as $y$ regarding the concerns of the DM.
\end{itemize}

\begin{example}{(Ex \ref{ex1}. Ctd.)}

Our aim in this paper is to build explanations, considering the case of binary reference scales (see Section \ref{sec_binary}). To stay within this framework and to demonstrate the feasibility of our proposal, we adapt the Example \ref{ex1} to this context. Therefore, the preference information is collected using alternatives on binary reference  scales.  Thus,  according to the Table \ref{matrix}, the preferential information $\mathcal{P}= \{ (e_1 \succsim e_2); (e_3 \succsim e_4), (e_4 \succsim e_5)\}$.

\begin{table}[htp!]
\begin{center}
\begin{tabular}{|c|c|c|c|c|}
\hline
    \emph{Criterion: $i$}  & \emph{\critun} & \emph{\critdeux} & \emph{\crittrois} & \emph{\critquatre}  \\ \hline
\emph{Top level:} $v_{i,2}$& \A :=$\AA$ &\B := $\BB$ & \C := $\CC$& \D:= $\DD$ \\ \hline
\emph{Bottom level}: $v_{i,1}$&\a :=  $\aa$&  \b :=$\bb$&  \c :=$\cc$&  \d := $\dd$\\ \hline
\end{tabular}
\end{center}
\caption{Criteria scales}
\label{crit}
\end{table}

\begin{table}[htp!]
\begin{center}
\begin{tabular}{|c|c|c|c|c|}
\hline
 & \emph{\critun} & \emph{\critdeux} & \emph{\crittrois} & \emph{\critquatre} \\ \hline
$e_1$& $\AA$ & $\bb$ & $\CC$& $\dd$\\ \hline
$e_2$ & $\aa$& $\BB$& \cc& $\DD$ \\ \hline
$e_3$ & $\aa$ & $\bb$ & $\cc$ & $\DD$\\ \hline
$e_4$ & $\aa$ &  $\BB$& $\CC$&  $\dd$\\ \hline
$e_5$ & $\AA$ & $\bb$&  $\cc$& $\dd$\\ \hline
\end{tabular}
\end{center}
\caption{Evaluation of the learning set}
\label{matrix}
\end{table}
\label{ex2}
\end{example}


\subsection{The even swaps method}
\label{sec_model2}

The even swaps method \cite{hamkeerai98}
relies on an additive utility function.  
In order to choose the best alternative,
this method does not require to fully elicit the marginal utility 
functions, but only a limited number of trade-offs between 
pairs of  attributes (swaps). In other words, the DM does
not have to explicitly define the preferences over the attributes in
general or to make any assumption about the form of the utility 
function. 
More precisely, the decision maker changes the consequence (or score) of an
alternative on one attribute, and compensates this change with a
preferentially equal change on another attribute. This creates a new
fictitious alternative, that is indifferent to the previous one,
with revised consequences. We use this alternative to try to
eliminate the other ones. The aim of this process is to carry out
even swaps that make either alternatives dominated or attributes
irrelevant.

Intuitively,  this method  can be seen as a scattered exploration of the isopreference curve (the curve where lies, even virtually, the alternatives equally preferred) of the DM. 
 This constructive method is quite intuitive as only two
attributes are involved in even swaps, and utilities are never explicitly mentioned to the DM.  Moreover, the existence of indifference statements required by the use of the even swaps method is guaranteed by the solvability condition used in conjoint measurement \cite[chapter 6]{krlusutv71}. 

We can also note that \cite{Mustajoki2009,Mustajoki2005} propose to enrich the original even swaps method in a way that accounts for incomplete knowledge about the value function. They consider a ``practical dominance'' notion when the value  of  an alternative is at least as high as the value of another one with every feasible combination of parameters, this perspective being  very close to the one developed in \cite{grmasl08} (see next section). However, this notion is only used for pre-processing dominated alternatives, and not integrated in the swap process, let alone used for explanatory purposes.

%
%

\subsection{Robust relation with the additive utility Model}

\label{sec_model3}

As already mentioned, the main drawback of the UTA method is that the preferential information brought by the  decision maker is not sufficient to uniquely specify the utility functions. 
A conservative approach consists in relying on a robust (or necessary) preference relation \cite{grmasl08,grslfimo09}. 
In words, the relation holds if \emph{any} possible completion of the available preferential information yields the preferential statement.

\begin{definition}[Necessary Preference Relation \cite{grmasl08,grslfimo09}]
 \label{defNPR}
 {\ } \newline

$x\succsim y$ iff $ \displaystyle \sum_{i\in N} u_i(x_i) \geq \sum_{i\in N}
u_i(y_i)$ for all $u_i$ s.t.
\begin{itemize}


\item $u_i$ is a monotonically increasing function $X_i
\rightarrow \reel$,

\item $\displaystyle \sum_{i\in N} u_i(a_i) \geq \sum_{i\in N}
u_i(b_i) \qquad \forall [a \succsim b]\in \mathcal{P}$.
\item $\displaystyle u_i(x_i^{min}), \forall i \in N, \sum_{i \in N} u_i(x_i^{max}) = 1$ where $x_i^{min} (x_i^{max}$ resp.) represent the minimum (maximum, respectively) envaluation in $X_i$. 

\end{itemize}
\end{definition}

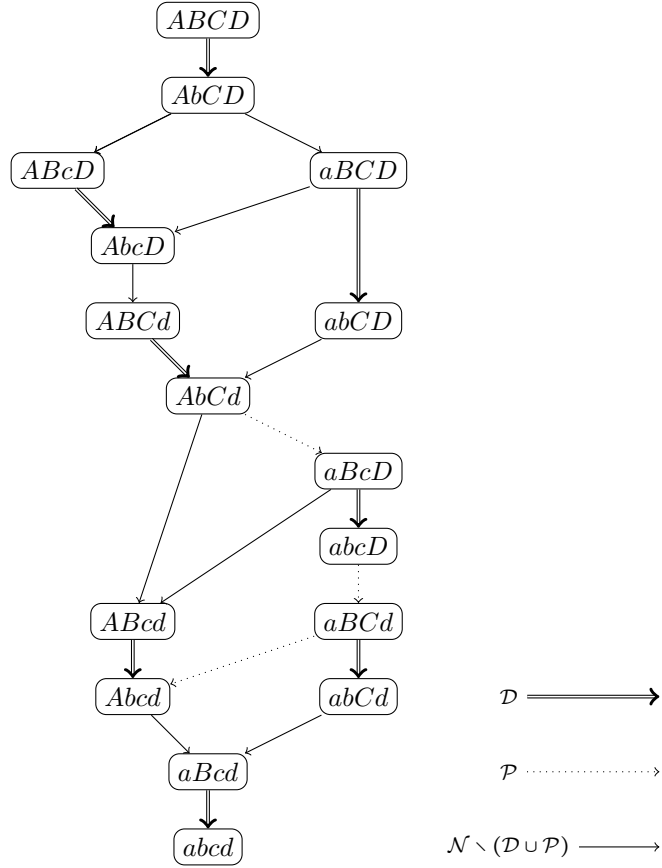
\begin{figure}[ht!]
\begin{center}

\begin{tikzpicture}
\path (0,12) node(ABCD) [rectangle, rounded corners, draw] {$\A\B\C\D$}
(0,11) node(AbCD) [rectangle, rounded corners, draw] {$\A\b\C\D$}
(-2,10) node(ABcD) [rectangle, rounded corners, draw] {$\A\B\c\D$}
(2,10) node(aBCD) [rectangle, rounded corners, draw] {$\a\B\C\D$}
(-1,9) node(AbcD) [rectangle, rounded corners, draw] {$\A\b\c\D$}
(-1,8) node(ABCd) [rectangle, rounded corners, draw] {$\A\B\C\d$}
(2,8) node(abCD) [rectangle, rounded corners, draw] {$\a\b\C\D$}
(0,7) node(AbCd) [rectangle, rounded corners, draw] {$\A\b\C\d$}
(2,6) node(aBcD) [rectangle, rounded corners, draw] {$\a\B\c\D$}
(2,5) node(abcD) [rectangle, rounded corners, draw] {$\a\b\c\D$}
(2,4) node(aBCd) [rectangle, rounded corners, draw] {$\a\B\C\d$}
(2,3) node(abCd) [rectangle, rounded corners, draw] {$\a\b\C\d$}
(0,2) node(aBcd) [rectangle, rounded corners, draw] {$\a\B\c\d$}
(0,1) node(abcd) [rectangle, rounded corners, draw] {$\a\b\c\d$}
(-1,4) node(ABcd) [rectangle, rounded corners, draw] {$\A\B\c\d$}
(-1,3) node(Abcd) [rectangle, rounded corners, draw] {$\A\b\c\d$}
(4,3) node(legendD) {\footnotesize{$\mathcal{D}$}}
(4,1) node(legendN) {\footnotesize{$\mathcal{N \setminus (\mathcal{D} \cup \mathcal{P})}$}}
(4,2) node(legendP) {\footnotesize{$\mathcal{P}$}}
; 

\tikzstyle{pareto}=[->, double]
\tikzstyle{domi}=[->]
\tikzstyle{pi}=[->, dotted]

\draw (ABCD) [pareto] -- (AbCD); 
\draw (ABcD) [pareto] -- (AbcD); 
\draw (aBCD) [pareto] -- (abCD); 
\draw (ABCd) [pareto] -- (AbCd); 
\draw (aBcD) [pareto] -- (abcD); 
\draw (aBCd) [pareto] -- (abCd); 
\draw (ABcd) [pareto] -- (Abcd); 
\draw (aBcd) [pareto] -- (abcd);

\draw (aBCD) [domi] -- (AbcD); 
\draw (AbCD) [domi] -- (ABcD); 
\draw (AbCD) [domi] -- (aBCD); 
\draw (AbcD) [domi] -- (ABCd); 
\draw (abCD) [domi] -- (AbCd); 
\draw (AbCd) [domi] -- (ABcd); 
\draw (AbCD) [domi] -- (ABcD); 
\draw (Abcd) [domi] -- (aBcd);
\draw (abCd) [domi] -- (aBcd); 
\draw (aBcD) [domi] -- (ABcd);

\draw (AbCd) [pi] -- (aBcD); 
\draw (abcD) [pi] -- (aBCd); 
\draw (aBCd) [pi] -- (Abcd); 

\draw (legendP) [pi] -- (6,2); 
\draw (legendD) [pareto] -- (6,3); 
\draw (legendN) [domi] -- (6,1);

\end{tikzpicture}

\end{center}
\caption{Necessary Preference relations}
\label{fig_necessary}
\end{figure}

Note that, once the DM has provided his preferential information $\mathcal{P}$, the reference scales $V_i$ are simply taken as the values that the alternatives in the preferential information use:
\begin{equation}
 V_i =\{ \{ a_i,b_i\} \: , \: (a \succsim b) \in \mathcal{P} \} .
\label{EqVi}
\end{equation} 

According to Definition \ref{defNPR}, we introduce a new relation: $\mathcal{N}$, the necessary preference
relation that can be deduced from $\mathcal{P}$ under the
assumption of an additive utility model (see the following section
for more details). $\mathcal{N}$ contains
$\mathcal{P}$,$\mathcal{D}$ and other pairs $(x,y)$ such that $x$
\emph{is necessarily preferred to} $y$, which can be expressed by
the \emph{statement} $(x \succsim y)$. For example, the graph of the Figure  \ref{fig_necessary} represents the necessary relation $\mathcal{N}$ build according to the Example \ref{ex2}. 

\section{Characterization of the robust value based preference relation}

\label{sec_charac}

Our aim through this paper is to provide a solid mechanism to construct explanations for necessary preference relations. Before a new explanation framework is introduced and discussed in Section \ref{sec_explaining}, this section presents some new algorithmic results concerning the necessary preference relation. Such results are important to establish some properties of the explanation engine detailed in Section \ref{sec_binary}.

\subsection{Rounding queries to the reference scales}

Definition \ref{defNPR} enables to express queries where
alternatives are taken in the whole evaluation space $X=X_1 \times
X_2 \times \ldots \times X_n$. However, on each criterion $i\in
N$, preferential information $\mathcal{P}$ is expressed on a
subset of $X_i$, the \emph{reference scale} $V_i$ (see \refe{EqVi}).

When $x_i, y_i \in V_i$ for all $i\in N$ the pessimistic
evaluation for all monotonically increasing functions can be
written as a minimization problem, where the objective function
and the constraints are linear functions of the $\displaystyle
\sum_{i\in N}p_i$ variables $u_i({v}_{i,k}), i\in N, 1 \leq k
\leq p_i$ \cite{grmasl08,grslfimo09}:
$$ (x \succsim_? y)\in \mathcal{N}$$
$$\Updownarrow$$
\begin{equation}\min \sum_{i\in N} u_i(x_i) - u_i(y_i) \geq 0
\label{NecAsMin}\end{equation}
$$
\textrm{s.t.}\left\{\begin{array}{lll}

 u_i(v_{i,k})-u_i(v_{i,k'}) & \geq 0 & \forall i\in N,\;\forall k,k' \;,\; 1 \leq k' < k \leq p_i\\
 \displaystyle \sum_{i\in N} u_i(a_i)-u_i(b_i)  & \geq
  0 &  \forall (a\succsim b) \in \mathcal{P}\\

\end{array}\right.$$

In order to account for values $x_i, y_i$ outside $V_i$, previous
studies \cite{grmasl08,grslfimo09} augment the reference scales $\widehat{V_i} := V_i \cup
\{x_i,y_i\}$, thus allowing the resolution space to evolve
dynamically according to the input alternative. We propose a
characterization of the necessary relation as linear programs with
variables and constraints defined statically by the preferential
information $\mathcal{P}$. Alternatives evaluated outside the
\emph{reference scales}  are preprocessed and fitted onto $V$
using a conservative rule, directly accounting for the worst case
under monotonicity constraints. This fitting proceeds on each
criterion $i\in N$ separately, as permitted by the additive
utility model. On criterion $i$, the values of $\widehat{V_i}$ are
sorted in ascending order, inserting $x_i$ and $y_i$ into the
scale $v_{i,1} < v_{i,2} < \ldots < v_{i,p_i}$. Three exclusive cases arise, as it is illustrated in the Figure \ref{fig_1} :
\begin{itemize}
\item If $y_i$ is strictly at the top of the scale, and/or $x_i$
is strictly at the bottom of the scale (i.e. $x_i < y_i$ and [
$y_i
> v_{i,p_i}$ or $x_i < v_{i,1}$]), then 
necessary preference of $x$ over $y$ is  impossible, and we say
the query $(x \succsim_? y)$ is \emph{unbounded} by $\mathcal{P}$ (see case (a) of Figure \ref{fig_1}).

\item If the pair $y_i \leq x_i$ is adjacent (i.e. $y_i \leq x_i$
and $\nexists k \; y_i \leq v_{i,k} \leq x_i$), then criterion $i$ is
neutral in the eventual necessary preference of $x$ over $y$. 
The specific value of attributes $x_i,y_i$ does not matter, and can
be 
replaced by the wildcard $\star$ (see criterion $k$ in case (b) of Figure \ref{fig_1}).

\item Otherwise, there are indexes $k, k'$ such that $v_{i,k} \leq
x_i$ and $ y_i \leq v_{i,k'}$ are adjacent pairs. We will show that the eventual
necessary preference of $x$ over $y$ can be examined by
considering the reference values $(v_{i,k},v_{i,k'})$ instead of
$(x_i, y_i)$. In this substitution, the candidate to preference
$x$ is rounded down to the next lower value in $V_i$, while the
challenger $y$ is rounded up (see criterion $i$ and $j$ in case (b) of Figure \ref{fig_1}).

These cases are formalized by the following Definitions \ref{defbounded} and \ref{defrounding}.

\end{itemize}


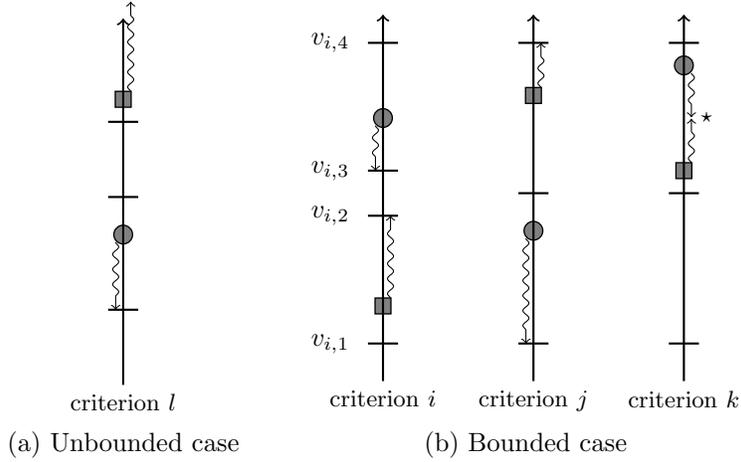
\begin{figure}[ht!]
\begin{center}
\begin{tabular}{ccc}
\begin{tikzpicture}

\tikzstyle{alt-O}=[circle,inner sep = 2.5pt,draw, fill=black!50, line width=0.5pt]
\tikzstyle{alt-X}=[rectangle, inner sep = 3pt, draw, fill=black!50, line width=0.5pt]

\path 
node at (12,3.8) [alt-X] {}
node at (12,2) [alt-O] {};
\path node (bottom3) at (12,0)  [below] {\small{criterion $l$}}
node (top3) at (12,5) {}; 
\draw (11.8,3.5) [thick] -- (12.2,3.5);
\draw (11.8,2.5) [thick] -- (12.2,2.5);
\draw (11.8,1) [thick] -- (12.2,1);
\draw [->,decorate,decoration={snake,amplitude=.4mm,segment length=2mm,post length=1mm}] (12.1,3.9) -- (12.1,5.1);
\draw [->,decorate,decoration={snake,amplitude=.4mm,segment length=2mm,post length=1mm}] (11.9,1.9) -- (11.9,1);
\draw [->,thick] (bottom3) -- (top3);

\end{tikzpicture}
 &  &

 \begin{tikzpicture}
\tikzstyle{alt-O}=[circle,inner sep = 2.5pt,draw, fill=black!50, line width=0.5pt]
\tikzstyle{alt-X}=[rectangle, inner sep = 3pt, draw, fill=black!50, line width=0.5pt]

\path 
node at  ( 2,2.8)  [alt-X] {}
node at ( 2,4.2)  [alt-O] {}
node at ( 2.3,3.5) {$\star$}
node at ( 0,3.8) [alt-X] {}
node at ( 0,2)  [alt-O] {}
node at ( -2,3.5)  [alt-O] {}
node at  (-2,1)  [alt-X] {}
node at (-2.7,0.5) {$v_{i,1}$}
node at (-2.7,2.2) {$v_{i,2}$}
node at (-2.7,2.8) {$v_{i,3}$}
node at (-2.7,4.5) {$v_{i,4}$};
\draw [->,decorate,decoration={snake,amplitude=.4mm,segment length=2mm,post length=1mm}] (-1.9,1.1) -- (-1.9,2.2); 
\draw [->,decorate,decoration={snake,amplitude=.4mm,segment length=2mm,post length=1mm}] (-2.1,3.4) -- (-2.1,2.8); 
\draw [->,decorate,decoration={snake,amplitude=.4mm,segment length=2mm,post length=1mm}] (-0.1,1.9) -- (-0.1,0.5); 
\draw [->,decorate,decoration={snake,amplitude=.4mm,segment length=2mm,post length=1mm}] (0.1,3.9) -- (0.1,4.5); 
\draw [->,decorate,decoration={snake,amplitude=.4mm,segment length=2mm,post length=1mm}] (2.1,4.1) -- (2.1,3.51); 
\draw [->,decorate,decoration={snake,amplitude=.4mm,segment length=2mm,post length=1mm}] (2.1,2.9) -- (2.1,3.49); 
\path node (bottom0) at (-2,0) [below] {\small{criterion $i$}}
node (top0) at (-2,5) {}; 
\path node (bottom1) at (0,0) [below] {\small{criterion $j$}}
node (top1) at (0,5) {}; 
\path node (bottom2) at (2,0) [below]  {\small{criterion $k$}}
node (top2) at (2,5) {}; 

\draw (-2.2,0.5) [thick] -- (-1.8,0.5);
\draw (-2.2,2.2) [thick] -- (-1.8,2.2);
\draw (-2.2,2.8) [thick] -- (-1.8,2.8);
\draw (-2.2,4.5) [thick] -- (-1.8,4.5);

\draw (-0.2,0.5) [thick] -- (0.2,0.5);
\draw (-0.2,2.5) [thick] -- (0.2,2.5);
\draw (-0.2,4.5) [thick] -- (0.2,4.5);

\draw (1.8,0.5) [thick] -- (2.2,0.5);
\draw (1.8,2.5) [thick] -- (2.2,2.5);
\draw (1.8,4.5) [thick] -- (2.2,4.5);

\draw [->,thick] (bottom0) -- (top0);
\draw [->,thick] (bottom1) -- (top1);
\draw [->,thick] (bottom2) -- (top2);
\end{tikzpicture} \\
(a) Unbounded case & & (b) Bounded case   \\
\end{tabular}

\caption{Rounding of a query $(x \succsim_? y)$, where $\bullet$ and $\filledsquare$ represent respectively the value of option $x$ and of option $y$.}
\label{fig_1}
\end{center}
\end{figure}

\begin{definition}[bounded and unbounded preference queries]
\label{defbounded} {\ }\newline Let $x,y\in X$, we say the
preference query $(x \succsim_? y)$ is \emph{unbounded by
$\mathcal{P}$} if $\exists i \in N, y_i > x_i$ and
($y_i>v_{i,p_i}$ or $x_i < v_{i,1})$. Otherwise, the query $(x
\succsim_? y)$ is \emph{bounded by $\mathcal{P}$}.
\end{definition}

\begin{definition}[Rounding of bounded preference queries]
\label{defrounding} {\ }\newline
Let $(x \succsim_? y)$ a bounded query. It is mapped to the
rounded query, $[x \succsim_? y] := (\underline{x} \succsim_?
\overline{y})$, where alternatives $(\underline{x},\overline{y})$
are taken in the reference scale $V$ and defined jointly as
follows :


$ \forall i\in N,(\underline{x}_i,\overline{y}_i)=\left\{
    \begin{array}{l} (\star,\star) \textrm{  , if   }x_i \geq y_i\textrm{  and  }\nexists k \;,\; y_i \leq v_{i,k} \leq x_i \\
    (\max \{v\in V_i, v \leq x_i\},\min \{v\in V_i, v\geq y_i\})\textrm{, else.} \\ \end{array} \right.$

\end{definition}

Note that queries expressed on alternatives taken inside the reference scale $V$ are left unmodified when rounded.

\begin{theorem}[Reduction to the reference scale]
\label{rouding} {\ }\newline Let $x,y\in X$. $x$ is necessarily
preferred to $y$ iff the query $(x \succsim_? y)$ is bounded by
$\mathcal{P}$ and $[ x\succsim_? y]\in\mathcal{N}$.
\end{theorem}

\begin{proof}
If the query $(x \succsim_? y)$ is unbounded by $\mathcal{P}$,
then the LP (\ref{NecAsMin}) is unbounded, as the difference
$u_i(y_i) - u_i(x_i)$ is left unbounded by the constraints, and
can be made as large as needed to ensure $\sum_{i\in N} u_i(y_i) >
\sum_{i\in N} u_i(x_i)$. Thus, the LP is unfeasible and $(x
\succsim_? y)\notin \mathcal{N}$.

Else, suppose $\underline{x}$ is necessarily preferred to
$\overline{y}$. The proxy alternatives
$\underline{x},\overline{y}$ are so defined that $x$ dominates
$\underline{x}$ and $y$ is dominated by $\overline{y}$. Hence,
transitivity of the relation $\mathcal{N}$ ensures that $x$ is
necessarily preferred to $y$.

Reciprocally, suppose $\underline{x}$ is not necessarily preferred
to $\overline{y}$ : there is a vector $u = (u_i: V_i \rightarrow
\reel)_{i\in N}$ of increasing functions compatible with the
preferences $\mathcal{P}$ ranking $\underline{x}$ lower than
$\overline{y}$. If required, these functions $(u_i)$ can be
extended, by specifying the images of $(x_i)$ and/or $(y_i)$ when
they are not in the reference scales  $(V_i)$, into another vector
$(\widehat{u}_i)$ of increasing functions compatible with the
preferences $\mathcal{P}$ such that
$\widehat{u}(x)=\widehat{u}(\underline{x})<\widehat{u}(\overline{y})=\widehat{u}(y)$,
i.e. $x$ is not necessarily preferred to $y$. If
$(\underline{x}_i,\overline{y}_i) \neq (\star,\star)$, let
$\widehat{u}_i(x_i) := u_i(\underline{x}_i)$ and
$\widehat{u}_i(y_i) := u_i(\overline{y}_i)$. These definitions
enforce the monotonicity of $\widehat{u}_i$.
If $(\underline{x}_i,\overline{y}_i) = (\star,\star)$, then
$\widehat{u}_i(x_i)=\widehat{u}_i(y_i)$ can be assigned any value
preserving the monotonicity of $\widehat{u}$.

\end{proof}

 \subsection{Covector associated with a query}
The characterization of necessary preferences can be further streamlined. The objective function and the constraints can be expressed as functions of the \emph{elementary preferences} $w_{i,k}:=u_i(v_{i,k+1})-u_i(v_{i,k})$ indexed by $I=\{(i,k), i\in N, 1 \leq k <
p_i\}$ :

$$ \left.\begin{array}{ll}\forall i \in N \; , \; \forall t \in V_i \;,\; u_i(t) &\displaystyle = u_i(v_{i,1}) + \sum_{\begin{array}{c}\scriptstyle 1\  \leq \ k \  <\ p_i \\[-5pt] \scriptstyle v_{i,k+1}\ \leq \ t \\ \end{array}}\left(u_i(v_{i,k+1})-u_i(v_{i,k})\right)\\[+7pt]  & \displaystyle = u_i(v_{i,1}) + \sum_{\begin{array}{c}\scriptstyle 1 \ \leq \ k \ <\ p_i \\[-5pt] \scriptstyle v_{i,k+1}\ \leq \ t \\ \end{array}}w_{i,k}\end{array}\right.$$


Hence,

$$ \forall i \in N \; , \; \forall x_i,y_i \in V_i \;,\; u_i(x_i)-u_i(y_i) = \sum_{\begin{array}{c}\scriptstyle 1\  \leq\  k\  <\ p_i \\[-5pt] \scriptstyle v_{i,k+1}\ \leq \ x_i \\ \end{array}} w_{i,k} - \sum_{\begin{array}{c}\scriptstyle 1 \ \leq \ k\  <\ p_i \\[-5pt] \scriptstyle v_{i,k+1}\ \leq \ y_i \\ \end{array}}w_{i,k}$$

Consider the contribution of the elementary preference $w_{i,k}$
to the right hand side : $w_{i,k}$ contributes in favor of
alternative $x$ with coefficient $+1$ iff $x_i \geq v_{i,k+1} >
v_{i,k} \geq y_i$, in favor of $y$ with coefficient $-1$ iff $y_i
\geq v_{i,k+1} > v_{i,k} \geq x_i$, and does not contribute (with
coefficient $0$) otherwise. Consequently, we introduce the
covector form of a query (for illustration see Figure \ref{fig_cov}):

\begin{definition}[Covector associated to a preference query]
\label{defcovector} {\ }\newline
$$\left.\begin{array}{lll} V \times V &\rightarrow &\{-1,0,1\}^I\\ (x \succsim_? y) & \mapsto & (x\succsim_? y)^{\star} \\ \end{array}\right.$$
$$  (x\succsim_? y)^{\star}_{i,k}=\left\{
    \begin{array}{lll}
        +1 &  \mbox{if  }\;  x_i \geq v_{i,k+1}>v_{i,k}\geq y_i & (w_{i,k}\textrm{ is an argument supporting }x)\\
        -1 &  \mbox{if} \; y_i \geq v_{i,k+1}>v_{i,k}\geq x_i& (w_{i,k}\textrm{ is an argument supporting }y)\\
        0 & \mbox{else } & (w_{i,k}\textrm{ is a neutral argument})\\
    \end{array}
\right.$$

\end{definition}


The set of covectors associated to a dominance statement is
$\{0,1\}^I$. A dominance covector can be broken down as a sum of
elementary dominance covectors $d^\star_{i,k}$ for $(i,k)\in I$ \ :\
$(d^\star_{i,k})_{i',k'} = 1$ if $i=i'$ and $k=k'$, 0
otherwise. Actually, $(d^\star)=((1,0,\ldots,0), (0,1,0,\ldots,
0), \ldots, (0, \ldots, 0,1))$ is the canonical covector base.

The introduction of covectors yields to a concise formulation :

$$ \forall x,y \in V \;,\; \sum_{i\in N} u_i(x_i)-u_i(y_i) = \sum_{(i,k)\in I}(x \succsim_? y)^{\star}_{i,k} \ w_{i,k}$$

Finally, using the product notation between covectors and vectors of $\reel^I$ to omit indexes :

$$ \forall x,y \in V \;,\; u(x) - u(y) = (x \succsim_? y)^{\star} \ w$$

Consequently, the following theorem links the transitivity of the
relation $\mathcal{N}$ and the sum of covectors.

\begin{theorem}[Chasles relation for covectors]
\label{Chasles} {\ \newline} Let $x, y, z \in V$
$$ (x \succsim_? y)^\star + (y\succsim_? z)^\star = (x \succsim_?
z)^\star$$
\end{theorem}

\begin{figure}[ht!]
\begin{center}
\begin{tikzpicture}
\tikzstyle{alt-O}=[circle,inner sep = 2.5pt,draw, fill=black!50, line width=0.5pt]
\tikzstyle{alt-X}=[rectangle, inner sep = 3pt, draw, fill=black!50, line width=0.5pt]

\path 
node at  ( 2.2,3.5)  [] {$\star$}
node at ( 1.65,3.5)  [] {$w_{k,2}$}
node at ( 1.65,1.5)  [] {$w_{k,1}$}
node at ( -2.35,3.6)  [] {$w_{i,3}$}
node at ( -2.35,2.5)  [] {$w_{i,2}$}
node at ( -2.35,1.4)  [] {$w_{i,1}$}
node at ( 0,4.5) [alt-X] {}
node at ( 0,0.5)  [alt-O] {}
node at ( -0.35,3.5)  [] {$w_{j,2}$}
node at ( -0.35,1.5)  [] {$w_{j,1}$}
node at ( -2,2.8)  [alt-O] {}
node at  (-2,2.2)  [alt-X] {};
\%draw [->,decorate,decoration={snake,amplitude=.4mm,segment length=2mm,post length=1mm}] (2.1,2.9) -- (2.1,3.49); 
\path node (bottom0) at (-2,0) [below] {\small{criterion $i$}}
node (top0) at (-2,5) {}; 
\path node (bottom1) at (0,0) [below] {\small{criterion $j$}}
node (top1) at (0,5) {}; 
\path node (bottom2) at (2,0) [below]  {\small{criterion $k$}}
node (top2) at (2,5) {}; 

\draw (-2.2,0.5) [thick] -- (-1.8,0.5);
\draw (-2.2,2.2) [thick] -- (-1.8,2.2);
\draw (-2.2,2.8) [thick] -- (-1.8,2.8);
\draw (-2.2,4.5) [thick] -- (-1.8,4.5);

\draw[decorate,thick,decoration={brace,raise=0.1cm}]
(-1.8,4.5) -- (-1.8,2.8) node[right=0.15cm,pos=0.5] {$0$};

\draw[decorate,thick,decoration={brace,raise=0.1cm}]
(-1.8,2.2) -- (-1.8,0.5) node[right=0.15cm,pos=0.5] {$0$};

\draw[decorate,thick,decoration={brace,raise=0.1cm}]
(-1.8,2.8) -- (-1.8,2.2) node[right=0.15cm,pos=0.5] {$+1$};

\draw (-0.2,0.5) [thick] -- (0.2,0.5);
\draw (-0.2,2.5) [thick] -- (0.2,2.5);
\draw (-0.2,4.5) [thick] -- (0.2,4.5);

\draw[decorate,thick,decoration={brace,raise=0.1cm}]
(0.2,2.45) -- (0.2,0.5) node[right=0.15cm,pos=0.5] {$-1$};

\draw[decorate,thick,decoration={brace,raise=0.1cm}]
(0.2,4.5) -- (0.2,2.55) node[right=0.15cm,pos=0.5] {$-1$};

\draw (1.8,0.5) [thick] -- (2.2,0.5);
\draw (1.8,2.5) [thick] -- (2.2,2.5);
\draw (1.8,4.5) [thick] -- (2.2,4.5);

\draw[decorate,thick,decoration={brace,raise=0.1cm}]
(2.2,4.5) -- (2.2,2.5) node[right=0.15cm,pos=0.5] {$0$};
\draw[decorate,thick,decoration={brace,raise=0.1cm}]
(2.2,2.5) -- (2.2,0.5) node[right=0.15cm,pos=0.5] {$0$};

\draw [->,thick] (bottom0) -- (top0);
\draw [->,thick] (bottom1) -- (top1);
\draw [->,thick] (bottom2) -- (top2);
\end{tikzpicture}
\caption{Elementary preferences $w$ and associated covector coefficients, where $\bullet$ and  $\filledsquare$ represent respectively the values $\underline{x}$ and $\overline{y}$ on the criteria}.\label{fig_cov}
\end{center}

\end{figure}
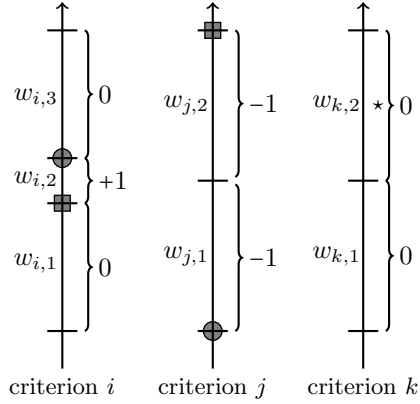


\subsection{Characterizations of the necessary preference relation}

In this subsection, we provide three equivalent linear programming formulations permitting to answer a necessary preference query. The first one is expressed as a condition on the elementary preferences, while the last two check a linear dependence relation between linear forms representing preferences.

\begin{theorem}[Primal atomic characterization of necessary preference]
\label{primal}
{\ }
\begin{center}
 $x$ is necessarily preferred to $y$
\end{center}
    $$\Updownarrow$$
$(x\succsim_? y)$ is bounded by $\mathcal{P}$ and $[x \succsim_?
y] ^{\star} \ w \geq 0$ for all $w\in\reel^I$ s.t.
$$\left\{
\begin{array}{ll}
 \displaystyle w_{i,k} \geq 0 & \forall (i,k)\in I \textrm{ (monotonicity) }
 \\  \displaystyle  (a \succsim b)^{\star} \ w \geq 0 & \forall (a \succsim b)\in \mathcal{P}.
\end{array}
\right.$$
\end{theorem}

Theorem \ref{primal} references $[x \succsim_? y]^{\star}$, the
covector associated to a bounded query after it has been rounded
to the reference space $V$. 
Such a  covector can be computed in two
phases, going successively through definitions \ref{defrounding}
and \ref{defcovector}. However, this process can be shortened. For
criterion $i\in N$:

\begin{itemize}
\item If $x_i \geq y_i$, only the intervals $[v_{i,k} ,v_{i,k+1}]$ fully included in the \emph{asset} $[x_i , y_i]$ are taken
into account positively. Other intervals are neutral.
\item If
$y_i
> x_i$, all the intervals $[v_{i,k} , v_{i,k+1}]$ necessary to
cover the \emph{liability} $]x_i , y_i[$ are taken into account
negatively. Other intervals are neutral. There is a possibility
for insufficient coverage, when $(x \succsim_? y)$ is unbounded by
$\mathcal{P}$.
\end{itemize}

\begin{theorem}[Covector associated to a rounded query]
{\ }\newline Let $(x \succsim_? y)$ a preference query bounded by
$\mathcal{P}$. The coefficients of the covector $[x \succsim_?
y]^{\star}$ representing its associated rounded query are given by
:

$${[x \succsim_?
y]^{\star}}_{i,k}= \left\{
    \begin{array}{ll}
        +1 &  \mbox{if  }\;  [v_{i,k} , v_{i,k+1}] \subset [y_i , x_i] \\
        -1 &  \mbox{if} \; [v_{i,k} , v_{i,k+1}] \cap ]x_i , y_i[ \ne \emptyset\\
        0 & \mbox{else } \\
    \end{array}
\right.$$

\end{theorem}

Establishing the preference of $x$ over $y$ can thus be seen as
covering the liabilities of $y$ by the assets of $x$. We call this
characterization atomic, as it proceeds without splitting the
fundamental resources $]v_{i,k}, v_{i,k+1}[$. When such a resource
would be split, it retains its full cost $w_{i,k}$ as a liability
but loses all its buying power as an asset. When the whole range
$]v_{i,1}, v_{i,p_i}[$ is not enough to meet the needed cost,
preference is impossible.


%

%

\begin{theorem}[Dual characterization of necessary preference, LP form]
\label{lp}{\ }
\begin{center}
 $x$ is necessarily preferred to $y$
\end{center}
    $$\Updownarrow$$
    the query $(x\succsim_? y)$ is bounded by $\mathcal{P}$ and there are real nonnegative
    coefficients $(\lambda)_{\mathcal{P}}$ and $(\mu)_I$ such that :
    $$ [x \succsim_? y]^{\star} = \sum_{(a \succsim b)\in \mathcal{P}}\lambda_{a,b}  (a \succsim
    b)^{\star} + \sum_{(i,k)\in I}\mu_{i,k} d^\star_{i,k} $$
    where the $(d^\star)_I$ are the elementary dominance covectors.
\end{theorem}

\begin{proof}
Thanks to Theorem \ref{primal}, the necessary preference of $x$
over $y$ is characterized by the inconsistency of the linear
system of inequations in the variable $w\in\reel^I$ : $$\left\{
\begin{array}{ll}
 \displaystyle -[x \succsim_? y]^{\star}\ w > 0 & \\
 \displaystyle d^\star_{i,k}\ w \geq 0 & \forall (i,k)\in I
 \\  \displaystyle  (a \succsim b)^{\star} \ w \geq 0 & \forall (a \succsim b)\in \mathcal{P}.\\
\end{array}
\right.$$ According to Farkas's lemma \cite{man69}, this inconsistency is
equivalent to the linear dependance of the covectors written in
the theorem.
%
%
%
%

\end{proof}

This LP formulation is used in Algorithm \ref{alg1} : \texttt{FindExplanation} described in Section \ref{sec_binary}.

Following Definition \ref{defNPR}, the necessary preference relation could
appear as a black box, as it interprets a minimization over a
functional space. From theorem \ref{lp}, this relation corresponds to  the
conical combination of the assertions expressed in $\mathcal{P}$
and of the elementary assertions derived from monotonicity (see also \cite{SplietTervonen14}). We
believe that actually strengthens the importance of the necessary
preference relation, as one of the most fundamental tools to
explore preference.

The linear combination can be constrained to use rational numbers only, yielding to an ILP formulation :
\begin{theorem}[dual characterization of necessary preference, ILP form]
\label{ilp}
{\ }\begin{center}
 $x$ is necessarily preferred to $y$\end{center}
   $$\Updownarrow$$
   the query $(x\succsim_? y)$ is bounded by $\mathcal{P}$ and there is a positive integer $r$ and non negative integers $(\ell)_{\mathcal{P}}$ and $(m)_I$ such that :
   $$ r [x \succsim_? y]^{\star} = \sum_{(a \succsim b)\in \mathcal{P}}\ell_{a,b}  (a \succsim
   b)^{\star} + \sum_{(i,k)\in I} m_{i,k} d^\star_{i,k} $$ 
   where the $(d^\star)_I$ are the elementary dominance covectors.
\end{theorem}

\begin{proof}
The coefficients $(\lambda)$ and $(\mu)$ of theorem \ref{lp} are
solutions of an affine system of equations with coefficents in
$\{-1,0,1\}$. The solution set of this system is non-empty in
$\reel$ iff it is non empty in $\rat$. Thus, let
$\lambda=\frac{\ell}{r}$ and $\mu=\frac{m}{r}$, fractions of
integers with common denominator $r$. Multiplying by $r$ the
relation with $\lambda,\mu$ yields the relation sought with $r,
\ell, m$.
\end{proof}

Note that the characterization proposed in the previous results is different from the characterization results shown for the additive utility model in conjoint measurement \cite[chapter 9]{krlusutv71}.
In Theorem \ref{ilp}, we provide necessary and sufficient condition for a given alternative $x$ to be necessarily preferred to an alternative $y$, while axiomatic results in decision theory are interested in finding necessary and sufficient conditions on the preference relation to be representable by a given model, where these conditions take the form of properties that the relation shall satisfy.
For instance the additive utility model is characterized by Archimedean property, double cancellation and independence, if the attributes are continuous \cite[chapter 6]{krlusutv71}.
However, this characterization does not apply to the necessary preference relation.

 Obviously, Theorem \ref{lp} is more convenient
than Theorem \ref{ilp} for implementation purpose. Nevertheless,
we believe the characterization given by theorem \ref{ilp} gives
an even deeper insight into the necessary preference relation.
Indeed, the query appears to be a linear combination with integer coefficients of the preferential information and of dominance.
This ILP form is leveraged to prove, in Section \ref{sec_charac}, a result on the length of the explanation in a particular case -- see Theorem \ref{tbt} (the proof is presented in  \ref{Proofs}).


\section{Explanation Engine}

\label{sec_explaining}

In this section, we present an original approach to explain a  robust preference. We recall that our decision model is based on a robust additive utility model (see Section \ref{sec_model}). A straightforward way of providing explanation would use the computed utility functions and exhibit them to the decision maker. 
The difficulty of this approach is twofold.
Firstly, utilities are basically non normalized as they return values in different scales, thus it is not always easy to interpret them.
One easy way to solve this issue is to impose that the extreme values on each attribute are commensurate , this is the fact that valuations on different criteria can be compared.
More precisely, criteria are supposed to be \emph{not satisfied at all} for every minimum element of the attributes, and criteria are supposed to be \emph{completely satisfied} for every maximum element of the attributes.
This amounts to transform the additive utility model into a weighted sum, where the utilities are normalized and a weight is assigned to each criterion.
One can then interpret the overall score in terms of the importance of criteria and the degree of satisfaction of criteria \cite{klein94,caremoor06,LabreucheAIJ2011}.
Once more, this conversion requires to add a commensurability assumption, which one may not accept.
In this latter situation, the utility functions in the additive utility model are meaningless to the decision maker.
Secondly, the previous process is not possible when the utility functions are not precisely known, which is the case with the robust preference relation. 

We propose to explore a different venue in order to provide explanations easily understandable by the DM : sequences of simple preference statements, that can be followed step by step to check the decision process. This section is dedicated to detail this notion of explanation.

\subsection{Explanations based on preference swaps}

Our idea of explanation is somewhat reminiscent of the even swaps method (see Section \ref{sec_model2}) which is an interactive process aimed at building a sequence of equally preferred (\emph{even}) alternatives, where only two attributes are modified at each step (\emph{swaps}),  initiated by the preferred term of the query  and the final term of the sequence dominates the less preferred term of the query. This constructive method is quite intuitive as only two attributes are involved in each even swap, and utilities are never explicitly mentioned to the decision maker.  In fact, if we consider the example displayed in \cite{hamkeerai98}, we can easily observe that from the reasoning steps of the even swaps process we can deduce an intuitive and simple manner to explain the result to the decision maker, without referring explicitly to the utility function or the model used to get the solution.
We keep the idea of explanations made of an easy to follow sequence of alternatives, but we have to take into account several obstacles to import the even swap process wholesale. While the even swap method involves an interactive process aimed at choosing the best alternative through a sequence of indifference statements (and deletion of non discriminating criteria), our decision model is more cautious (robust and non-parametric), and we aim at providing an explanation without further interaction with the decision maker. In this context of deliberately frugal information, exploring an isopreference curve of the decision maker by asking equivalence queries is clearly not on the table. Therefore, we relax the notion of equal preference, and only require the sequence of alternatives to be of decreasing preference, thus introducing the notion of \emph{preference swap}.

\begin{definition}{(Preference swap)}

When $(x \succsim y) \in \mathcal{N} \setminus \mathcal{D}$, we say that $x \succsim y$ is a \emph{preference-swap of
order $p$} with $p := |\{I \in N, x_i \neq y_i\}| \in \{2,\ldots,n\}$.
\label{swap_pref}
\end{definition}

As preference swaps exclude the dominance, $x$ is preferred to $y$ on at least one attribute, and $y$ is preferred to $x$ on at least one attribute.
We will consider the set $\Delta_p$ containing the pairs of alternatives $(x,y)$ such that $x \succsim y$ is  a preference swap of order $p$,  for $p \in \{2, \dots, n\}$. Clearly, $\mathcal{N}$ is partitioned between the sets $\mathcal{D}, \Delta_2, \dots ,\Delta_n$, whereas the preferential information $\mathcal{P}$ may contain preference swaps on any order.
We believe the order of a preference swap gives an indication about the cognitive complexity of a statement : when alternatives $x$ and $y$ differ on many attributes, the statement $(x \succsim y)$ is hard to understand.

From now on, we define an explanation as a chain of sequences such that each sequence $x^i \succsim x^{i+1}$ is either a preference swap of at most a given order or a dominance relation ($\mathcal{D}$) between the two alternatives. More precisely, this type of explanation transforms one single comparison $x\succsim y$ that the DM needs to understand by a sequence of several preferences $x^i\succsim x^{i+1}$.  The idea is that the initial preference $x\succsim y$ is complex to understand as the values of $x$ and $y$ differ on most  (if not all) attributes, whereas each intermediate comparison $x^i\succsim x^{i+1}$ is much easier to understand as $x^i$ and $x^{i+1}$ differ only on a few attributes.  In other terms,  if $x^i$ and $x^{i+1}$ have the same value on all attributes except  on a few ones, then the preference will be apparent for the DM. We see that the order of the preference swap helps to  better understand  the explanation.
 Thus, we can define an explanation as follows:


\begin{definition}{(Explanation--$\mathcal{E}_k$)}

An explanation of length $q$ of the necessary relation  $x \succsim y$ is a sequence $(x^1 \succsim x^2)
\dots (x^{q-1} \succsim x^q)$ of dominance relation and/or preference  swaps of order at most $k$  ($\mathcal{D} \cup \Delta_2 \cup \cdots \cup \Delta_k$) such that  $x^1= x$ and $x^q =y$ and $x^2, \dots, x^{q-1} \in X$.
\label{def_explain}
\end{definition}

According to Definion \ref{def_explain} the set of statements  that admit an explanation in $\mathcal{D} \cup \Delta_2 \cup \dots \cup \Delta_k$ is noted $\mathcal{E}_k$. 

\begin{remark}
\label{rem1}
We note that  if the definition of explainability by dominance relations and
preference swaps of order  at most $k$ allows the intermediate alternatives $x^2 \succsim \ldots \succsim x^{q-1}$ to be taken anywhere in space $X$, these alternatives can be restricted to the augmented reference
scales $\widehat{V}$ without loss of generality.

\end{remark}

\subsection{Low order preference swaps}

Even though we have defined $\mathcal{E}_k$ for a general index $k$, we will consider from now on only the case $k=2$. The reason is that preference swaps of order $2$ are much simpler to understand for the decision maker.
(However, one may use $\mathcal{E}_k$, with $k>2$, if there is no explanation in $\mathcal{E}_2$ of a given preference in $x \succsim y$.   Obviously,  $\mathcal{E}_2  \subseteq \mathcal{N}$). 


The concept of swaps between \emph{two} attributes is also known in engineering. For instance, the Architecture Tradeoff Analysis Method (ATAM) is used in order to assess software architectures according to ``quality attribute goals''  \cite{kaz-ATAM00}. 
A \emph{trade-off point} is an architecture parameter affecting at least two quality attributes in different directions. 
For example, increasing the speed of the communication channel improves throughput in the system but reduces its reliability.
Thus the speed of that channel is a trade-off point.
The concept of trade-off point in ATAM makes explicit the interdependencies between pairs of attributes.
The decision maker can express preferences among trade-offs from these dependencies.

From Definition  \ref{def_explain} (with $k=2$),  it is important to observe that technically  an explanation is a path  from $x$ to $y$ describing a  decreasing sequence of preferences in the directed graph $G:=( \widehat{V} ,  \mathcal{D} \cup \Delta_2)$, such that the vertices are in $\widehat{V}$ (which  has an exponential size with respect to the number of criteria) and the edges are preference swaps of order 2 or dominance relation (for more details see Section \ref{sec_binary}).  In what follows, we shall see that we can face some technical problems in order to find such a path.

\subsection{ Discussion on some technical challenges with explanation}

In this work we consider that the basic building blocks of an explanation is a preference swap of order 2. Before going any further on how to construct such an explanation, we shall discuss some points  that  we believe important  for the definition of an explanation.  In fact, a  primary question regarding an explanation is to be able to say if it is satisfactory or not for the decision maker. To account for that, we shall discuss several points related to an explanation, namely: existence, length, values of the attributes involved in the sequences and position of the different swaps.

These four points are important both from a theoretical point, a practical point of view and an implementation or operational point of view, as they may have an impact on the impression of the decision maker on the explanation. In order to construct a convincing explanation, the framework should be at the same time theoretically rigorous and pertinent from  a practical (implementation) perspective too. We draw the attention of the reader on the fact that it is outside the scope of the paper to resolve all the points discussed below. Our aim is to highlight some of the complexities around the question of constructing an explanation.

\begin{itemize}

\item  \emph{Existence of an explanation}

The first point to consider in the construction of an explanation is to make sure there is one to be found. Without any additional assumption, it is quite possible that $\mathcal{E}_2 \subsetneq \mathcal{N}$, that  there  are some statements that cannot be explained by preference swaps of order 2 and dominance relations.
%

For instance, in the context of Example \ref{ex2}, we have $(\A, \B, \C, \d) \succsim (\a, \b, \c, \D) \in \mathcal{N} \setminus \mathcal{E}_2$ (see Example \ref{ExNotInE2} in Section \ref{Structure}).

\begin{quote}
Is it possible, by imposing conditions on the preferential information $\mathcal{P}$, to ensure the existence of an explanation from every statement in $\mathcal{N}$? 
\end{quote}

Technically, checking if we can explain  a statement $x \succsim y$  in $\mathcal{E}_2 $, can be seen as determining if the vertices  $x$ and $y$ are \emph{connected} in the directed graph $G:=( \widehat{V} ,  \mathcal{D} \cup \Delta_2)$. Of course, we have efficient algorithms in order to test if a graph is connected or not  \cite{EvenT75}. However, it may be challenging to use them with regard to the size of the graph  in our context.

In Section \ref{sec_binary}, we come up with answers to these challenging questions, under the assumption of \emph{binary reference scale}. In this specific case, the existence of an explanation is guaranteed when preferences from $\mathcal{P}$ only refer to swaps of order 2, and can be efficiently checked otherwise.

\item \emph{Length of an explanation }

A second point that we address here is the length $q$ of the sequence. Indeed, keeping the explanation short has a great bearing on its ability to convince. Even if each elementary comparison $x^i \succsim x^{i+1}$ is trivial for the decision maker, the overall sequence $x \succsim x^2 \succsim \ldots \succsim x^{q-1} \succsim y$ cannot be seen as a convincing explanation if its size exceeds a given value. One then looks for explanations with the smallest possible size.

%

Finding the shortest explanation means  resolving the problem of \emph{shortest path} in the directed graph $G:=( \widehat{V} ,  \mathcal{D} \cup \Delta_2)$.  Thus, the length of a shortest explanation is bounded by the \emph{diameter}\footnote{The diameter of a graph $G$ is the longest distance between two vertices in graph.}
of the graph $G$. Finding such a diameter is a classical problem in graph theory for which we have polynomial algorithm in terms of number of vertices and edges (see for instance \cite{Aingworth:1996}) . However, it can be too slow to be practical in our case with $|\widehat{V} |$ vertices (exponential with respect to the number of criteria).

In what follows, we come up with two results concerning the length of the explanations provided by our framework.
First, without additional assumptions, the length of the explanation is unbounded, as soon as there are 3 criteria or more.

\begin{theorem}[Unbounded length of shortest explanations $\mathcal{E}_2$]
\label{Plength1}{\ }\newline Consider $n = 3$.
For every $p\in \ent^*$, if $X_1 \supseteq \{0,1,2,\ldots,2p\}$,  $X_2 \supseteq \{-p,-p+1,\ldots,-1,0\}$ and $X_3 \supseteq \{-p,-p+1,\ldots,-1,0\}$,
then there exists some preferential information $\mathcal{P}$ and $x,y \in X$ such that
$x \succsim y$ and the minimal length of the explanation in $\mathcal{D} \cup \Delta_2$ is at least $2p$.
\end{theorem}

\begin{figure}[ht!]
\begin{center}
\begin{tikzpicture}
\begin{axis}[ view={120}{40}, width=220pt, height=220pt, grid=major, z buffer=sort, xmin=-4,xmax=0, ymin=-4,ymax=0, zmin=0,zmax=4, enlargelimits=false, xtick={-5,...,0}, ytick={-5,...,0}, ztick={0,...,5}, xlabel={$X_3$}, ylabel={$X_2$}, zlabel={$X_1$}, fill=gray!20 ]
\addplot3+[fill opacity=0.3, only marks,scatter,mark=cube*,mark size=20, cube/size x =22, cube/size y =29] coordinates {
(0,-0.2,1)
(-1,-0.2,2)
(0,-1.2,2)
(-2,-0.2,3)
(-1,-1.2,3)
(0,-2.2,3)
(-3,-0.2,4)
(-2,-1.2,4)
(-1,-2.2,4)
(0,-3.2,4)};
\addplot3+[thick,black, solid, line width=2pt, mark=circle] coordinates {(0,0,0) (0,-1,1)};
\addplot3+[thick,black, solid, line width=2pt, mark=circle] coordinates {(0,-1,1) (-1,-1,2)};
\addplot3+[thick,black, solid, line width=2pt, mark=circle] coordinates {(-1,-1,2) (-1,-2,3) };
\addplot3+[thick,black, solid, line width=2pt, mark=circle] coordinates {(-1,-2,3)  (-2,-2,4)};
\end{axis}
\end{tikzpicture}
\end{center}
\caption{Description of the sequence}\label{Fig3D}
\end{figure}
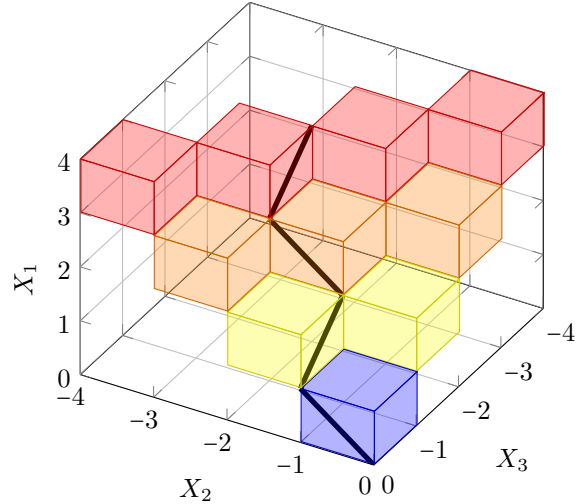

The proof of this result can be found in \ref{Proofs}.

For the sketch of the proof, we construct, for every $p$, a preference between $x=(0,0,0)$ and $y=(2p ,-p,-p)$.
Starting from alternative $(0,0,0)$, we begin with a preference swap between attributes $1$ and $2$ (adding value $1$ on the first attribute, and subtracting $1$ on the second one).
Then we perform a preference swap between attributes $1$ and $3$ (adding value $1$ on the first attribute, and subtracting $1$ on the third one).
We proceed then again by a preference swap between attributes $1$ and $2$, and so on (the sequence is depicted in Figure \ref{Fig3D}).

Second, under the assumption of \emph{binary reference scales}, restricting the number of values that the preferential information can take on each attribute, we provide a tight upper bound on the length of explanations in Section \ref{sec_binary}.

\item  \emph{Values of the terms in the sequence}

Another point concerns the choice of the values of the intermediate alternatives $x^2, \ldots, x^{q-1}$ on the different attributes. In fact, if these values are all different and not chosen appropriately, this can induce a cognitive load to the DM, when analyzing the sequence $x \succsim x^2 \succsim \ldots \succsim x^{q-1} \succsim y$.
Many different choices of the values of the intermediate alternatives $x^2, \ldots, x^{q-1}$ on the different attributes can be considered. A first option is that the value of these alternatives on each attribute can be only the value of $x$ or $y$.  We think that this case is suitable if we need to explain only one comparison $x\succsim y$. However, this is not always the case, and one can ask for explanations for several comparisons $x\succsim y$, $x'\succsim y'$, $x''\succsim y''$, $\ldots$. At this moment,   the intermediate alternatives used for the different explanations will be quite different using the first option. In the second option,  the value of these alternatives on each attribute can only take predefined values (where the predefined list of values does not depend on $x$ and $y$).  
Hence the intermediate alternatives contained in the explanation of $x\succsim y$, $x'\succsim y'$, $\ldots$ use the same values on the attributes, which reduces the workload for the DM.
Some policy concerning the preferential use of some particular values in $\widehat{V}$, favoring the values $V$ of the preferential information, or the values of the attributes of the alternatives being compared, can prove more or less convincing, depending on the context (see Section \ref{sec_binary}).

\item  \emph{Swaps positions}

Finally, another important aspect is  the position of the swaps in the sequence. In other words,  we believe that the ordering of the swaps in the explanation may have an impact on the persuasiveness of the explanation.  Various policies can be considered, for instance \emph{Strongest first}, where the strongest support (attribute) is presented first, in order to
get early on  at least a provisional agreement from the decision maker In this case, we should be able to determine which attribute is the most convincing for the DM, or \emph{Safer first} presenting first the attributes that refer directly to the preferential information given by the DM, ...
This has some connection with rethoric where an important aspect is to determine the order in which the arguments are to be presented to the audience, depending on their sign and strength \cite{maygold96}.

Technically, let us consider an explanation as a sequence of swaps, rather than a sequence of alternatives. 
As a preference swap specifies that an increase on a  given criterion is compensated by a decrease on another criterion (according to the preferences of the decision maker), regardless of the values of the other attributes,  it is quite possible that the same swap appears at different places in the graph in order to connect the vertices  $(x^1 \succsim  y^1), (x^2 \succsim y^2), \dots$. A remarkable property of the sequence of swaps is the possibility to commute between swaps related to different criteria : let the explanation $x \succsim y \succsim z$ with the swaps $s_1=(x \succsim y)$ and $s_2= (y \succsim z)$ that are not related to the same criteria, then we also have an explanation $x \succsim y \succsim z$ involving $s_2$ followed by $s_1$. This ability to change the order of  the swaps inside an explanation suggests some freedom to present  the  different arguments in an order that  will help to increase the persuasiveness of an explanation. Furthermore, when the swaps share a criterion, which is increased by one and decreased by another one (for instance, the first swap compensates an increase on $i$ by a decrease on $j$ and  the second swap compensates an increase on $j$ by a decrease on $k$), it might be possible to trade directly between criteria $i$ and $k$, without involving $j$. We apply this principle of reduction by transitivity, under the assumption of binary reference scales, in Section \ref{sec_binary}.

\end{itemize}


\section{Properties of explanations when preferences are expressed on binary reference scales}

\label{sec_binary}

This section aims at deepening the understanding of the
explanation engine introduced in section \ref{sec_explaining}, by focusing on
the case where preferences only refer to two distinct values of
the attributes on each scale. We will refer to this assumption as
\emph{binary reference scales}. Section \ref{BinaryDiscussion}
discusses the meaning and relevance of this assumption. Section
\ref{Arguments} introduces the partition of criteria into arguments supportive or not of a preference query, and leverages the assumption of binary reference scales to simplify indexes, covectors, and rounding of queries introduced in section
scales. Section \ref{Structure} reveals the core term-by-term
structure of any explanation, and resulting properties. 

\subsection{Binary reference scales}{\label{BinaryDiscussion}

\emph{Binary reference scales} are encountered when the preferences $\mathcal{P}$ expressed by the decision maker only reference two levels on each attribute : $\forall i \in N , V_i = \{\bot_i < \top_i\}$. Besides luck, such a tight reference set is the consequence of one of these two situations :
\begin{itemize}
\item \emph{attributes are themselves binary} : present or absent features, passed or failed checks, etc. Also, such binary attributes may result from any model relying on subset comparisons. While they fall outside the scope of this article, we believe the explanation engine discussed here can address problems not necessarily resulting from an additive utility decision model (for instance, robust weighted majority decision models rely on subset comparisons between coalition of criteria, as do pan-balance comparisons encountered in extensive measurement problems).
    \item 
    \emph{when expressing preference statements, the decision maker is deliberately restricted to comparing between prototypical alternatives specifically chosen in $\prod_{i\in N} \{\bot_i,\top_i\}$}. This process is supposed to help the decision maker focus on the main aspects of the preference problems, by limiting the number of moving parts between alternatives, and by referring to carefully chosen reference values, serving as anchors. This technique is used in the field of experimental design (yielding the one-factor-at-a-time or the factorial experiments methods), as well as in multicriteria decision aiding. For instance, the MACBETH method \cite{BeCV95,Costa2008} is based on binary alternatives : in order to assess hidden technical parameters (the weights of the various criteria), the decision maker is asked to express preference between prototypical alternatives, traditionally referencing a neutral level $\bot_i$ (for technological products, representing the attribute of a mid-range, available product), and a high level $\top_i$ (representing the attribute of a luxury product, or a hypothetical performance demanding a technological breakthrough). Note that, while MACBETH assumes commensurability between the neutral levels $(\bot_i)_{i\in N}$ and between the high levels $(\top_i)_{i\in N}$ (representing the satisfying reference level, in the sense of Simon \cite{sim56}, commensurability is neither needed nor assumed in this article. Not that in MACBETH, the two reference are called Low and High and are not necessarily assumed to be commensurate. 
    \end{itemize}


\subsection{Positive and negative arguments of a
statement}{\label{Arguments}} 

When preferences reference exactly
two values on each criterion, i.e. $V_i=\{\bot_i,\top_i\}$, the
index set boils down to the criteria set : $I = N \times \{1\} =
N$. There are one-to-one correspondences between the set $V\times
V$ of queries $q$ expressed on the reference scales, the set
$\{-1,0,1\}^N$ of covectors $q^\star$, and the partitions of $N$
between positive $q^+$, negative $q^-$ and neutral $q^0$
arguments (see Example \ref{ex_argu}). Let $(x\succsim_? y)\in V\times V$ :
\begin{itemize}
\item positive arguments : $$(x \succsim_? y)^+ = \{i\in N,
x_i=\top_i \textrm{ and } y_i=\bot_i\} = \{i\in N, (x \succsim_?
y)^\star = +1\}$$ \item negative arguments : $$(x \succsim_? y)^-
= \{i\in N, x_i=\bot_i \textrm{ and } y_i=\top_i\} = \{i\in N, (x
\succsim_? y)^\star = -1\}$$ \item neutral arguments : $$(x
\succsim_? y)^0 = \{i\in N, x_i=y_i\} = \{i\in N, (x \succsim_?
y)^\star = 0\}$$
\end{itemize}

\begin{example} (Ex \ref{ex1}. ctd.)
For the statement $(x \succsim y)$ introduced in Example \ref{ex1}, positive arguments $(x \succsim y)^+ = \{\crittrois,\critquatre\}$, negative arguments $(x \succsim y)^-=\{\critun, \critdeux\}$, and there is no neutral argument.
\label{ex_argu}
\end{example}

As a shortcut, when $N^+$ and $N^-$ are two disjoint subsets of
$N$, we denote $(N^+ \succsim N^-)$ the query for which positive and
negative arguments are respectively $N^+$ and $N^-$. Its
associated covector is $(N^+ \succsim N^-)^\star = \sum_{i\in
N^+}d^\star_i - \sum_{i\in N^-}d^\star_i$.

Dominance $\mathcal{D}$ is characterized by statements without
negative arguments $(N^+ \succsim \emptyset)$. Associated
covectors are in $\{0,1\}^N$ and can be broken down as a sum of
elementary dominance covectors $\sum_{i\in N^+}d^\star_i$.

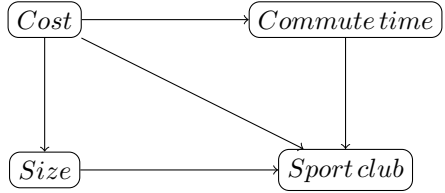
\begin{figure}[ht!]
\begin{center}
\begin{tikzpicture}
\tikzstyle{match}=[->, double]
\tikzstyle{normal}=[->]

\path (-2,2) node(crit4) [rectangle, rounded corners, draw] {$\critquatre$};
\path (2,2) node(crit1) [rectangle, rounded corners, draw] {$\critun$};
\path (-2,0) node(crit3) [rectangle, rounded corners, draw] {$\crittrois$};
\path (2,0) node(crit2) [rectangle, rounded corners, draw] {$\critdeux$};

\draw (crit4) [normal] -- (crit1); 
\draw (crit4) [normal] -- (crit3); 
\draw (crit3) [normal] -- (crit2); 
\draw (crit4) [normal] -- (crit2); 
\draw (crit1) [normal] -- (crit2); 

%
%
%
%
%
%

\end{tikzpicture}
\end{center}
\caption{Binary relation between criteria}
\label{fig_nec_reduc}
\end{figure}

A preference swap of order 2 has exactly one positive argument $i$
and one negative argument $j$ : $(\{i\} \succsim \{j\})$. It
expresses the preference of a raise of attribute $i$ from the
bottom to the top of the scale $V_i$, over a corresponding decrease
of attribute $j$, or the acceptance to compensate a nominal
decrease of attribute $j$ with a nominal increase of attribute
$i$. The set $\Delta_2$ of preference swaps of order 2 can be
viewed as a binary relation between criteria, and represented by a
graph with nodes in $N$ : $(i \succsim j)$. This idea is illustrated on the Figure \ref{fig_nec_reduc} and Example \ref{ex_binary_crit}

\begin{example} (Ex \ref{ex1}.  and Ex \ref{ex2}. ctd.)
The necessary preference relation deduced from the preferential information given in Example \ref{ex2} contains the following preference swap of order 2, represented on Figure \ref{fig_nec_reduc}.  

Thus, $\Delta_2=$$\{(\critun \succsim \critdeux),$ $(\crittrois \succsim \critdeux),$ $(\critquatre \succsim \critun)$ $, (\critquatre \succsim \critdeux) $$,(\critquatre \succsim \crittrois)\}$. For instance, the abbreviated statement $(\critquatre \succsim \critdeux)$, represented by the arrow from $\critquatre$ to $\critdeux$, means that an alternative ranking higher than $\D$ on attribute $\critquatre$ and low on attribute $\critdeux$ is necessarily preferred to one ranking low on
$\critquatre$ (between $\d$ and $\D$) and high on $\critdeux$, attributes $\critun$ and $\crittrois$ being equal : $(\star, \b, \star, \D) \succsim (\star, \B , \star, \d)$
\label{ex_binary_crit}
\end{example}

Generally, a statement $(N^+ \succsim N^-)$ is a preference swap
of order $|N^+|+|N^-|$.

When considering a query $(x
\succsim_? y)$ and a criterion $i$ for which preferences
$\mathcal{P}$ only reference two values $V_i = \{\bot_i <
\top_i\}$, the rounding rules presented in section \ref{sec_charac} boil down
to 5 mutually exclusive cases for attributes $x_i, y_i$ :
\begin{itemize} \item $i$ is a strong argument supporting $x$, as
$x_i \ge \top_i > \bot_i \ge y_i$. Then, $(x \succsim_? y)^\star_i
= +1$ ; \item $i$ is a weak argument supporting $x$, as $x_i >
y_i$ but $[\bot_i , \top_i] \nsubseteq [y_i , x_i]$. Then, $(x
\succsim_? y)^\star_i = 0$ ; \item $i$ is a truly neutral
argument, as $x_i = y_i$. Then, $(x \succsim_? y)^\star_i = 0$ ;
\item $i$ is a weak argument supporting $y$, as $\top_i \ge y_i >
x_i \ge \bot_i$. Then, $(x \succsim_? y)^\star_i = -1$ ; \item $i$
is a strong argument supporting $y$, as either $y_i$ is greater
than both $x_i$ and $\top_i$, or $x_i$ is smaller than both $y_i$
and $\bot_i$. Then, $(x \succsim_? y)\notin \mathcal{N}$ and the
covector $(x \succsim_? y)^\star$ is not defined.
\end{itemize}

\subsection{Term by term explanations}{\label{Structure}}
 The following theorem reveals the core structure every explanation is built upon.
\begin{theorem}[Term by term explanation]
\label{tbt}{\ }\newline Let a statement ${\sigma}\in (V\times V)
\cap \mathcal{N}$, and $\sigma^+,\sigma^-$ be the positive and negative arguments of $\sigma$ respectively. The following propositions are equivalent :
\begin{enumerate}[(i)]
\item ${\sigma} \in \mathcal{E}_2$ \item $\exists a \in
\ent^\star, \gamma_1, \cdots, \gamma_q \in \Delta_2, \ell_1,
\cdots, \ell_q \in \ent, m_1, \cdots , m_n \in \ent :$
$$ a{\sigma}^\star = \sum_k \ell_k \gamma^\star_k + \sum_k
m_k d^\star_k $$ \item There is a matching of cardinality
$|{\sigma}^-|$ in the graph of $\Delta_2 \cap ({\sigma}^+ \times
{\sigma}^-)$. \item There is an injection $\phi : {\sigma}^-
\rightarrow {\sigma}^+$ such that $ \forall k \in {\sigma}^-,
(\{\phi(k)\} \succsim \{k\})$.
\end{enumerate}

\end{theorem}
In a nutshell, an explanation is a sequence where, at each step, a positive argument is used up to cancel an inferior negative argument, and, eventually, every negative argument has been cancelled.
We highlight three consequences of this theorem :
\begin{itemize}
\item \emph{If preferences only refer to swaps of order 2, then every necessary preference can be explained by swaps of order 2.}

The next result shows that $\mathcal{N}=\mathcal{E}_2$. 
This is a potent existence result for explanations, and it provides a complete description of the necessary preference relation under the assumption of the decision maker expressing preferences between alternatives differing along two criteria only.
Compared to theorem \ref{ilp}, we do not restrict ourself to preference statements in $V\times V$.

\begin{corollary}[Existence of an explanation in $\mathcal{E}_2$]
 In the case of binary reference scales and $\mathcal{P} \subseteq \Delta_2$ then  $\mathcal{N}= \mathcal{E}_2$; i.e:  for any statement $(x\succsim y) \in \mathcal{N}$, there exists an explanation in $\mathcal{E}_2$ of $(x\succsim y)$.
\label{Cor1}
\end{corollary}

\proof
Let $s=(x\succsim y) \in \mathcal{N}$.
By theorem \ref{ilp}, $\sigma := [x \succsim_? y]$ meets condition (ii).
Then condition (i) of theorem \ref{tbt} holds for $\sigma$. As $(x,\underline{x}) , (\overline{y},y) \in \mathcal{D}$, we also have $s \in \mathcal{E}_2$.
\endproof

\item \emph{Explanations can be kept short.} The next corollary proves that the size of the explanation is at most $n/2$, which appears manageable for the recipient of explanation.

\begin{corollary}[Length of the explanation]
 For any statement $(x\succsim y) \in \mathcal{N}$, there exists an explanation with a length at most $\lfloor \frac{n}{2} \rfloor+2$, 
where $\left\lfloor m \right\rfloor$ denotes the integer part of $m$.
\label{Cor2}
\end{corollary}

\item \emph{Building an explanation, or ensuring there is none, is handled by an efficient  algorithm (see Algorithm \ref{alg1}). } A quick inspection of the complexity reveals that in the first part of the algorithm, there are at most $\mathcal{O}(n^2)$ calls to a linear program (with $n$ the number of criteria). 
This is followed by the resolution of a matching problem, which runs in its simpler version in $\mathcal{O}(n^3)$. 
Note that in theory, the number of constraints and variables of the LP may be exponential in $n$, because of the number of preference statements can be. In practice, this is of course highly unrealistic as it is too demanding for the DM. And for a polynomially bounded number of preference queries, the algorithm is efficient. 
\end{itemize}

\begin{center}
\begin{algorithm}[ht!]
\begin{footnotesize}
\caption{\sc{FindExplanation}} \label{alg1} 

\KwData{a statement $\sigma = (x \succsim y)$ to be explained, a set of
preference statements $\mathcal{P}$.}

\KwResult{a matching of each negative argument by a stronger
positive one.}

Compute $\sigma^+, \sigma^-$\\

\If{$|\sigma^+|<|\sigma^-|$}{return None}

\If{$\sigma \notin\mathcal{N}$}{return None}

Build the graph of $\Delta_2 \cap (\sigma^+ \times \sigma^-)$ :\\

 Initialize $\mathcal{G}$ as a graph with nodes $\sigma^+ \cup \sigma^-$ and no
 edge.\\

\For{$i \in \sigma^+$ }{

\For{$j\in \sigma^-$}{

\If{the LP with $|N|+|\mathcal{P}|$ inequality constraints,
$|N|$ equality constraints and $|N|+|\mathcal{P}|$ variables \\
$~~~\forall p\in\mathcal{P}, \ell_p \ge 0$\\
                $~~~\forall k\in N, m_k \ge 0$ \\
                $~~~\forall k\in N, \sum_{p\in\mathcal{P}}\ell_p \ p^\star_k + m_k = 1$ if $k=i$, -1 if $k=j$, 0
                else.\\
           is feasible}{add edge $(i,j)$					 to $\mathcal{G}$}

}
 }

Find a matching $\phi$ of maximum cardinality $C$ in bipartite
graph $\mathcal{G}$. \\

\If{$C < |\sigma^-|$}{return None}

{return $\phi$}

\end{footnotesize}

\end{algorithm}

\end{center}

\begin{figure}[ht!]
\begin{center}
\begin{tikzpicture}
\tikzstyle{match}=[->, double]
\tikzstyle{normal}=[->]

%
%
%

\path (6,2.5) node(arg) [] {\emph{positive arguments}};
\path (10,2.5) node(arg) [] {\emph{negative arguments}};

\path (6,2) node(crit4) [rectangle, rounded corners, draw] {$\critquatre$};
\path (10,2) node(crit1) [rectangle, rounded corners, draw] {$\critun$};
\path (6,0) node(crit3) [rectangle, rounded corners, draw] {$\crittrois$};
\path (10,0) node(crit2) [rectangle, rounded corners, draw] {$\critdeux$};

\draw (crit4) [match] -- (crit1); 
\draw (crit3) [match] -- (crit2); 
\draw (crit4) [normal] -- (crit2); 

\draw (8,0) [dotted] -- (8,2);

\end{tikzpicture}
\end{center}
\caption{Matching returned by Algorithm \ref{alg1} with data of Example \ref{ex2} }
\label{fig_matching}
\end{figure}
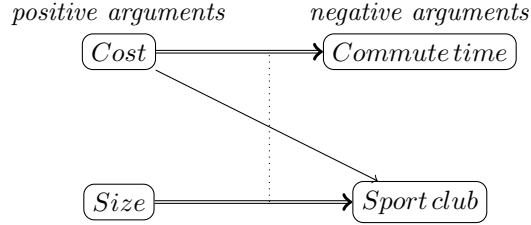


If the input statement $(x \succsim y)$ is in $\mathcal{E}_2$, the algorithm returns a mapping $\phi$ matching each negative argument by a stronger positive argument. $\phi$ is not an explanation \emph{per se}. In order to provide a suitable sequence, the elementary links - swaps $( \phi(k) \succsim k)_{k \in (x \succsim y)^-}$, and maybe some dominance relations - need to be sorted. Depending on the context of the decision, some policy concerning this sorting could prove more efficient than others, and this area of freedom left by the explanation engine should be investigated, theoretically and/or empirically, before any actual implementation.

There is also some leeway concerning the values of the attributes appearing in the explanation sequence. Consider the following policies :
\begin{itemize}
\item generate the shortest possible explanation, by grouping every dominance relations into one single step, and providing a sequence where every term is described by attributes similar to either $x$ or $y$.
This is depicted in the sequence \refe{eqsequence};

\begin{equation}
\left\{ \begin{array}{l}
  x^1 = x \\
  x^2 = (y_{\{\phi(j_1),j_1\}},x_{-\{\phi(j_1),j_1\}}) \\
  x^3 = (y_{\{\phi(j_1),j_1,\phi(j_2),j_2\}},x_{-\{\phi(j_1),j_1,\phi(j_2),j_2\}}) \\
  \cdots \\
  x^{p+1} = (y_{\{\phi(j_1),j_1,\ldots,\phi(j_p),j_p\}},x_{-\{\phi(j_1),j_1,\ldots,\phi(j_p),j_p\}}) \\
  x^{p+2} = y
	\end{array} \right. 
\label{eqsequence}
\end{equation}
\item generate an explanation one step longer, by having both an initial and final dominance steps, and providing a sequence where every term is described by attributes on the reference scales derived from $\mathcal{P}$, yielding the following sequence:
\begin{equation}
\left\{ \begin{array}{l}
  x^1 = x \\
	x^2 = \underline{x} \\
  x^2 = (\overline{y}_{\{\phi(j_1),j_1\}},\underline{x}_{-\{\phi(j_1),j_1\}}) \\
  x^3 = (\overline{y}_{\{\phi(j_1),j_1,\phi(j_2),j_2\}},\underline{x}_{-\{\phi(j_1),j_1,\phi(j_2),j_2\}}) \\
  \cdots \\
 x^{p+1} = (\overline{y}_{\{\phi(j_1),j_1,\ldots,\phi(j_p),j_p\}},\underline{x}_{-\{\phi(j_1),j_1,\ldots,\phi(j_p),j_p\}}) \\
  x^{p+2} = y
	\end{array} \right. 
\label{eqsequence2}
\end{equation}

\end{itemize}
The first policy seems a better fit for one-shot explanations, as it provides a shorter explanation looking a lot like the input statement. The second policy would be preferred for batch explanations, as all the explanations provided would share a common foundation of prototypical alternatives.

\begin{example}{(Ex\ref{ex1} and Ex \ref{ex2}. ctd.)}

Figure \ref{fig_matching} shows the bipartite graph of the relation $\Delta_2$ restricted to pairs of positive-negative arguments. The double arrows highlight a matching 
of cardinality 2, covering the negative arguments, as returned by Algorithm \ref{alg1} : $\{(\critquatre \succsim \critun), (\crittrois \succsim \critdeux)\}$. Therefore, the statement $(x \succsim y)$ can be explained by a sequence of preference swaps of order 2 and dominance relations. We propose 4 explanations, with $x$=(-45 min, no gym, 450 $m^2$, -5000 \euro) and $y$=(-15 min, gym, 180 $m^2$, -12500 \euro) :

\begin{itemize}
\item $x\ \Delta_2$ (-15 min, no gym, 450 $m^2$, -12500 \euro) $\Delta_2 \ y$
\item 
$x \ \Delta_2$ (-45 min, gym, 180 $m^2$, -5000 \euro) $ \Delta_2 \ y$
\item 
$x \ \mathcal{D} \  (\a,\b,\C,\D) \ \Delta_2 \ (\A,\b,\C,\d) \ \Delta_2 \  (\A,\B,\c,\d) \ \mathcal{D} \ y$
\item 
$x \ \mathcal{D} \  (\a,\b,\C,\D) \ \Delta_2 \ (\a,\B,\c,\D) \ \Delta_2 \  (\A,\B,\c,\d) \ \mathcal{D} \ y$
\end{itemize}
\end{example}


\begin{example}{(Ex\ref{ex1} and Ex \ref{ex2}. ctd.)}

From Figure \ref{fig_necessary}, one can readily see that $(\A, \B, \C, \d) \succsim (\a, \b, \c, \D) \in \mathcal{N}$.
Let us now try to find an explanation in $\mathcal{E}_2$.
The positive and negative arguments are 
$\sigma^+ = \{ \critquatre \}$ and $\sigma^-=\{ \critun, \critdeux, \crittrois\}$. Considering $\Delta_2$ as described in Example  \ref{ex_binary_crit}, we can see that 
condition (iii) of theorem \ref{tbt} is not satisfied, and thus $(\A, \B, \C, \d) \succsim (\a, \b, \c, \D) \in \mathcal{N} \setminus \mathcal{E}_2$.

\label{ExNotInE2}
\end{example}


\section{Conclusion}

\label{sec_conclu}

Generating explanations to justify recommendation is a key challenge to decision-aiding systems. 
While we witness the emergence of highly sophisticated methods to elicit preferences and compute recommended alternatives, the question of explanation is often neglected. We believe this may hinder the development of such systems. As a matter of fact, real decision makers prefer the use of a very basic model if its outcomes are transparent, rather that elaborate models that look as black box for them. They need explanation to accept the decision support recommendation. 

In this work we address this question in what is arguably the most basic model used in MCDM: a simple additive model, which is given here a robust interpretation when preferences are incomplete (as is usually the case in practice). Due to the presence of such imprecise utility functions, we propose an original approach where explanations are conceived as sequences of simple steps (simple comparisons of options). To be of practical interest, such sequences must exist, be of reasonable length, and each step should be easily checked by the DM. 
We have thus focused on these fundamental features in this paper.  
What we show in that respect is that the size of explanation depends on the number of values in the attribute space that are used in the preferential information. In particular, when there is no restriction on this number, the size of the explanations cannot be bounded. On the other hand, when the preferential information relies on two values per attribute only, we obtained a sharp bound on the size of a provably existing explanation. 

One question is whether our approach (in short: explanations conceived as sequences of simple steps) can be exported to other settings. While some general notions may be reused with different  preference models, we emphasize that some key properties of the additive model (in particular the ability to work on pairs of criteria, thanks to separability) are exploited in this work. 

Still, there are many possible extensions to this work. We cite some of the most prominent here. 

Firstly, there remain \emph{theoretical questions} to be studied. We have investigated two extreme cases:  in the first one, no assumption is made on preferential information (yielding a negative result in terms of the length of the explanation), while in the second one we assume a binary reference scale (and can guarantee the existence of a short explanation). 
A natural but challenging question is whether the complexity of the reference scale can be more generally linked to the size of the explanations. 

Secondly, we have provided an algorithm for the binary case only. It would be of practical interest to design and implement an algorithm finding the simplest (e.g. shortest) explanation in the general case.

Thirdly, while we discuss good theoretical properties of explanations, an \emph{empirical validation} remains to be conducted on other aspects mentioned (the sequencing of swaps, the choice of values, for instance). What makes the exercise difficult though is that this may highly depend on the context of use: a DM who needs to justify an important decision before a committee may not have the same expectations as a DM taking a decision for herself. 

Finally, the framework may be smoothly extended to cater for more general situations. For instance, the nature of the preferential information may be different. The DM may use a more expressive language, and give some statements on the intensity of their preferences. 
A first step in that direction is to assume a quaternary relation, of the form ``$o_1$ is \emph{more intensely} preferred to $o_2$ than $o_3$ is preferred to $o_4$''. 
While this would constitute a first step towards dealing with intensities, we are confident that this may still be handled within the framework described here.  
As a final suggestion on a possible extension of this framework, we note that this work makes the assumption that elicitation and explanation are dealt with separately.  
A certainly promising perspective is to extend the framework so that explanation and elicitation are actually intertwined. By putting forward an explanation, the system shows some evidence which can in turn trigger some reaction from the DM.


\bibliographystyle{spphys}       


%
%
%
%


\newpage

\appendix

\section{Proofs}
\label{Proofs}

\begin{proof}[Theorem \ref{Plength1}]

Let $n= 3$, $p\in \ent^*$.
Assume that  $X_1 \supseteq \{0,1,2,\ldots,2p\}$,  $X_2 \supseteq \{-p,-p+1,\ldots,-1,0\}$ and $X_3 \supseteq \{-p,-p+1,\ldots,-1,0\}$.
Consider the following preference information $\mathcal{P}$:

\begin{align}
 & \forall j\in\{0,\ldots,p-1\}  \nonumber \\
 & \quad \qquad (2j,-j,\cdot ) \succsim (2j+1,-j-1,\cdot )   \label{E21} \\
 & \forall j\in\{0,\ldots,p-1\}   \nonumber \\
 &  \quad \qquad (2j+1,\cdot,-j ) \succsim (2j+2,\cdot,-j-1 )   \label{E22} 
\end{align}

Hence $V_1 = \{0,1,2,\ldots,2p\}$,  $V_2 = \{-p,-p+1,\ldots,-1,0\}$ and $V_3 = \{-p,-p+1,\ldots,-1,0\}$.

We set $x=(0,0,0)$ and $y=(2p ,-p,-p)$.
With this $ \mathcal{P}$, we clearly obtain the sequence
\begin{align*}
  x = & (0,0,0) \succsim (1,-1,0 )  \qquad \mbox{(by \refe{E21})}\\
 & \succsim
 (2,-1,-1 ) \succsim \cdots  \qquad \mbox{(by \refe{E22})} \\
 & \succsim (2p-2,-(p-1),-(p-1) ) \\
 & \succsim (2p-1,-p,-(p-1))  \qquad \mbox{(by \refe{E21})} \\
 & \succsim (2p ,-p,-p) = y \qquad \mbox{(by \refe{E22})} \\
\end{align*}
so that $x \succsim y$.
This sequence is of length $2\,p$.

\bigskip

There remains to prove that this is the shortest explanation.

To this end, we first need to determine the form of $\Delta_2$. 
According to Theorem \ref{rouding}, we need only to consider the elements in$\Delta_2$ that belong to $V_1\times V_2 \times V_3$
(The other ones can be deduced by Pareto dominance).
The preference information \refe{E21} and \refe{E22} is very special.
In particular, any value $k \in V_1$ appears only in two examples -- one in which $k$ appears in the left hand side (in \refe{E21}) and the other one where $k$ appears in the right hand side (in (\ref{E22})).
Moreover, we notice that, in \refe{E21} and \refe{E22}, the value on the first attribute is always increasing from the left hand side to the right hand side,
and the value of the second and the third attributes is decreasing from the left hand side to the right hand side.
Hence the elements of $\Delta_2$ cannot be obtained by a combination of two or more preference information.
They are obtained only from one preference information (\refe{E21} and \refe{E22}) and Pareto dominance $\mathcal{D}$.
More precisely, $\Delta_2$ is composed of the following preferences
\begin{align*}
 & (i,j,k) \succsim (i',j',k')
\end{align*}
where either there exists $l$ such that $i=2l$, $j=2l+1$, $j \geq -l > -l-1 \geq j'$ and $k=k'$,
or there exists $l$ such that $i=2l+1$, $j=2l+2$, $j=j'$ and $k \geq -l > -l-1 \geq k'$.
From this, one can readily see that the explanation of the preference of $x$ over $y$ described earlier is the shortest one.

\end{proof}


\begin{proof}[ Theorem \ref{tbt} : term-by-term explanations]

We prove $(i) \Rightarrow (ii) \Rightarrow (iii) \Rightarrow (iv)
\Rightarrow (i)$ :

\begin{itemize}
\item $(i) \Rightarrow (ii)$ : Suppose there is a sequence $(x^1
\succsim x^2),\ldots,(x^{r-1} \succsim x^r)$ of dominance
relations and/or preference swaps of order 2 such that $x^1=x$ and
$x^r=y$. Without loss of generality, the terms of this sequence
can be taken on the reference scales : $x^k\in V$ (see remark \ref{rem1} in section \ref{sec_explaining}, with $k=2$). Chasles relation for covectors (th
\ref{Chasles}) applied to this sequence yields $\sum_k (x^k
\succsim x^{k+1})^\star = (x^1 \succsim x^r)^\star$, which in turn
yields the relation sought with $a=1$ after sorting the $(x^k
\succsim x^{k+1})$ between preference swaps and dominance
relations. Note that this covector transformation is a destructive
compression, as it does not preserve the sequential structure of
an explanation : there is no hope of restoring the original order
of the swaps $\gamma_k$. \item $(ii) \Rightarrow (iii)$ : Suppose
there exists integer coefficients $a, \ell_1, \cdots, \ell_q$, $m_1,
\cdots , m_n$ and preference swaps of order 2 : $\gamma_1,
\cdots, \gamma_q$ such that
\begin{equation} a{\sigma}^\star = \sum_k \ell_k
\gamma^\star_k + \sum_k m_k d^\star_k
\label{SumOfCovectors}\end{equation} Multiplying both sides of
covector equation (\ref{SumOfCovectors}) by the vector $(1,
\cdots, 1)$, we obtain the relation :
$$M := a(|{\sigma}^+| - |{\sigma}^-|)=\sum m_k \geq 0$$

To homogenize the right-hand side, we introduce a dummy criterion
0 standing for dominance : $d_k = ( {k} \succsim {0})$ and denote
$\widehat{N}=N\cup\{0\}$. Thus, relation $\mathcal{D}\cup\Delta_2$
is a graph with nodes in $\widehat{N}$. Re-indexing coefficients
$\ell_k$ by the positive and negative arguments of swap $\gamma_k$
(summing up duplicates if needed), and introducing $\ell_{k,0} :=
m_k$ :
\begin{equation}a \  {\sigma}^\star = \sum_{(i\  \succsim \ j)\in
\mathcal{D} \cap \Delta_2}\ell_{i,j} (i \succsim j)^\star
\label{eqfipl1}\end{equation}

In order to complete the flow $\ell$, we introduce :
\begin{itemize}
\item a source $s$ supplying flow $\ell_{s,i}=a$ to the positive
arguments $i\in{\sigma}^+$; \item a sink $t$ collecting flow
$\ell_{j,t}=a$ from the negative arguments $j\in{\sigma}^-$, and
$\ell_{0,t} = M$ from node 0.
\end{itemize}
Covector equation (\ref{eqfipl1}) ensures $\ell$ defines a
feasible flow on the graph
$(\widehat{N}\cup\{s,t\},\mathcal{D}\cup\Delta_2\cup \{s\}\times
{\sigma}^+ \cup {\sigma}^- \times\{t\}\cup\{(0,t)\})$, without
capacity constraints, as projection on the $i^{th}$ coordinate
ensures flow conservation for node $i\in N$. Flow $\ell$ can be
decomposed as a superposition of :
\begin{itemize}
\item cycles, involving necessary equivalence between the nodes,
and not contributing to the value of the flow; \item paths from
the source $s$ to the sink $t$ passing through node 0, denoting a
dominance relation. Their total contribution to the value of the
flow is $M$;
 \item paths from the
source $s$ to the sink $t$ not passing through node 0, with an
overall contribution of $a \times |{\sigma}^-|$ to the value of the flow.
Each of these paths links a positive argument $i_1\in {\sigma}^+$
to a negative argument $i_r \in {\sigma}^-$ through necessary
preference swaps of order 2. Transitivity of the necessary
preference relation entails that $i_1$ is necessarily preferred to
$i_r$ : the edge $(i_1,i_r)$ belongs to $\Delta_2 \cap ({\sigma}^+
\times {\sigma}^-)$.
\end{itemize}
We reduce flow $\ell$ by ignoring the cycles and paths passing
through node 0. Also, the flow $a$ carried by the path from source
to sink $ s \rightarrow i_1 \succsim i_2 \succsim \ldots \succsim
i_r \rightarrow t$ is redirected to edge $(i_1, i_r)$. As a
result, we obtain a flow of value $a|{\sigma}^-|$ on the graph of
the relation $\Delta_2$ restricted to ${\sigma}^+ \times
{\sigma}^-$. This entails the existence of a matching of
cardinality $|{\sigma}^-|$ in this graph, obtained by setting an
upper capacity constraint of value 1 on each edge leaving the
source $s$ and entering the sink $t$ (as a cut of capacity $C$ on
the network with capacity constraints $c_{i,j}\in\{1,\infty\}$ is
a cut of capacity $a\times C$ on the same network with capacity
constraints $a \times c_{i,j}$).

 \item
$(iii) \Rightarrow (iv)$ is simply a rewording.

\item $(iv) \Rightarrow (i)$ : Let $\phi : {\sigma}^- \rightarrow
{\sigma}^+$, injective, such that $ \forall k \in {\sigma}^-,
(\{\phi(k)\} \succsim \{k\})$. In subset form, the statement
$\sigma$ expressing the necessary preference of alternative $x \in
V$ over alternative $y\in V$ can be written $(\sigma^+ \succsim
\sigma^-)$. Given any ordering $\pi$ of the negative argument set
$\sigma^-$, we can build a sequence of alternatives of decreasing
preference $x_1 := x \succsim \ldots \succsim x_{|\sigma^-|+1}\in
V$ such that the $k^{th}$ statement $(x_k \succsim x_{k+1})$
matches the preference swap $(\phi(\pi_k) \succsim \pi_k)$. Thus,
the sequence of sets $(x_k \succsim y)^-$ decreases from
$\sigma^-$ to $\emptyset$, one element at a time, and the sequence
of sets $(x_k \succsim y)^+$ also decreases from $\sigma^+$ to
$\sigma^+ \setminus \phi[\sigma^-]$, one element at a time. If the
set $\sigma^+ \setminus \phi[\sigma^-]$ is empty,
$x_{|\sigma^-|+1} = y$, and the sequence $x = x^1 \succsim \ldots
\succsim x^{|\sigma^-|+1}=y$ is an explanation of $(x \succsim y)$
by preference swaps of order 2, of length $|\sigma^-|+1$. Else,
$x^{|\sigma^-|+1} \neq y$ but $(x^{|\sigma^-|+1} \succsim_? y)$ is
a dominance statement, as its negative argument set is empty.
Thus, the sequence $x = x^1 \succsim \ldots \succsim
x^{|\sigma^-|+1} \succsim y$ is an explanation of $(x \succsim y)$
by preference swaps of order 2 and a dominance relation, of length
$|\sigma^-|+2$.
\end{itemize}

\end{proof}

\end{document}